\documentclass{article}

\usepackage{arxiv}

\usepackage[utf8]{inputenc} 
\usepackage[T1]{fontenc}    
\usepackage{hyperref}       
\usepackage{url}            
\usepackage{booktabs}       
\usepackage{amsfonts}       
\usepackage{nicefrac}       
\usepackage{microtype}      
\usepackage{lipsum}
\usepackage{lscape}
\usepackage{natbib} 
\usepackage{booktabs}   
\usepackage{siunitx}
\usepackage{makeidx}
\usepackage{subcaption}
\usepackage{array} 
\usepackage{amssymb}
\usepackage{graphicx}
\usepackage{enumerate} 
\usepackage{amsmath} 
\usepackage{bigstrut}
\usepackage{amsopn}
\usepackage{algorithm}
\usepackage{algorithmicx}
\usepackage{algpseudocode}
\usepackage{float}
\usepackage{bm} 
\usepackage{rotating}
\usepackage{multicol}
\usepackage{multirow} 
\usepackage{tabularx}
\usepackage{arydshln} 
\usepackage{amsthm}
\newtheorem{theorem}{Theorem}
\newtheorem{lemma}{Lemma}
 \newtheorem{definition}{Definition}

\theoremstyle{definition}
\newtheorem{observation}[theorem]{Observation}
\usepackage{hyperref}
\hypersetup{
  colorlinks,
  allcolors=blue,
  urlcolor=blue,
}
\usepackage{color}
\usepackage[misc]{ifsym}
 
\usepackage{blindtext}
\usepackage{scrextend}

\title{CDF Transform-and-Shift: An effective way to deal with datasets of inhomogeneous cluster densities}

\author{
  Ye Zhu$^{ (\textrm{\Letter})}$
  \\
  School of Information Technology\\
  Deakin University\\
  Victoria, Australia 3125 \\
  \texttt{ye.zhu@ieee.org} \\ 
   \And
  Kai Ming Ting\\
  National Key Laboratory for Novel Software Technology\\
  Nanjing University\\
  Nanjing, China 210023 \\
  \texttt{tingkm@nju.edu.cn} \\ 
     \And
  Mark~J.~Carman\\
  DEIB\\  
  Politecnico di Milano\\
  Milan, Italy 20133\\
  \texttt{mark.carman@polimi.it} \\  
     \And
  Maia Angelova\\
  School of Information Technology\\
  Deakin University\\
  Victoria, Australia 3125 \\
  \texttt{maia.a@deakin.edu.au} \\ 
}

\begin{document}
\maketitle

\begin{abstract}
The problem of inhomogeneous cluster densities has been a long-standing issue for distance-based and density-based algorithms in clustering and anomaly detection. These algorithms implicitly assume that all clusters have approximately the same density. As a result, they often exhibit a bias towards dense clusters in the presence of sparse clusters. Many remedies have been suggested; yet, we show that they are partial solutions which do not address the issue satisfactorily. 
To match the implicit assumption, we propose to transform a given dataset such that the transformed clusters have approximately the same density while all regions of {\em locally} low density become {\em globally} low density---homogenising cluster density while preserving the cluster structure of the dataset.
We show that this can be achieved by using a new multi-dimensional Cumulative Distribution Function in a transform-and-shift method. The method can be applied to every dataset, before the dataset is used in many existing algorithms to match their implicit assumption without algorithmic modification. We show that the proposed method performs better than existing remedies.  
\end{abstract}


\section{Introduction}
\label{sec:introduction}
 
Many existing distance-based and density-based algorithms in clustering and anomaly detection have an implicit assumption that all clusters have approximately the same density. As a result,
these algorithms have a bias towards dense clusters and produce poor performance when presented with datasets having inhomogeneous cluster densities 
\cite{aggarwal2013data}. For example,  DBSCAN \cite{ester1996density} is biased towards grouping neighbouring dense clusters into a single cluster, in the presence of a sparse cluster \cite{ZHU2016983}.

A number of methods have been proposed to ``correct'' this bias, but many of them are tailored to specific algorithms. For example, in the context of clustering, 
to overcome the issue of DBSCAN \cite{ester1996density} that only uses a global density threshold to identify high-density clusters, OPTICS \cite{ankerst1999optics} utilises multiple density thresholds to extract clusters with varied densities based on a ``reachability'' (${k}^{th}$-nearest neighbour distance) plot. SNN clustering \cite{ertoz2003finding} proposes adaptive similarity measures in place of the distance measure in DBSCAN. In addition, DP \cite{rodriguez2014clustering} identifies cluster centres which have local maximum density and are well separated, and then assigns each remaining point to its cluster centre via a linking scheme.  

However, these methods are not satisfactory in addressing the density bias issue. The ranking on ``reachability'' plot of OPTICS depends heavily on the original data density distribution and the plot may merge different clusters which are close and dense \cite{ZHU2016983}. The performance of SNN is sensitive to the $k$ setting in the ${k}^{th}$-nearest neighbour distance calculation that depends on the local data density \cite{tan2020mutual}. Although DP uses local maximum density to identify cluster centres, the centre selection process is based on the density ranking on a decision graph \cite{rodriguez2014clustering}. In a nutshell, their performances are still affected by data density variation between clusters. 

To demonstrate the shortcomings of these four density-based clustering algorithms, we apply them to an image segmentation task, and the data points are represented in the LAB space \cite{szeliski2010computer}. Table \ref{seg1} shows their best clustering outcomes on the image shown in Figure~\ref{lab:aa}. It can be seen that all these clustering algorithms cannot identify the sky, cloud and building well, due to the significant density variation in the LAB space.\footnote{In DBSCAN, all points that are not clustered are noise points because they have very low density. If using a lower density threshold, DBSCAN will merge Sky and Cloud into a single cluster.} 

\begin{figure}[!p]
\centering 
   \centering\captionsetup{width=1\linewidth}%
     \includegraphics[width=0.37\textwidth]{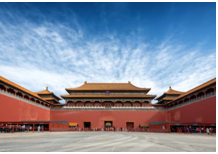} 
   \caption{An image used for segmentation.} 
        \label{lab:aa} 
\end{figure}

\begin{table} [!p] 
  \renewcommand{\arraystretch}{1.2}
 \setlength{\tabcolsep}{3.5pt}
\centering
\caption{Image segmentation on the image shown in Figure~\ref{lab:aa}. The colours in the scatter plot indicate different clusters identified by a clustering algorithm. The results only display the best-matched clusters wrt sky, cloud and buildings, i.e., the grey points in LAB scatter plots are labelled as noise or unmatched clusters by the clustering algorithm. We searched different parameters in a reasonable range. The bandwidth of the density estimator used in DBSCAN and DP is set to $\epsilon=0.1$.} 
  \begin{tabular}{|c|c|ccc|}
    \hline
     & LAB & Cluster 1  & Cluster 2   & Cluster 3  \\
      \hline
       \begin{turn}{90}  \ \ \   DBSCAN \end{turn}&     
       \includegraphics[width=1.37in]{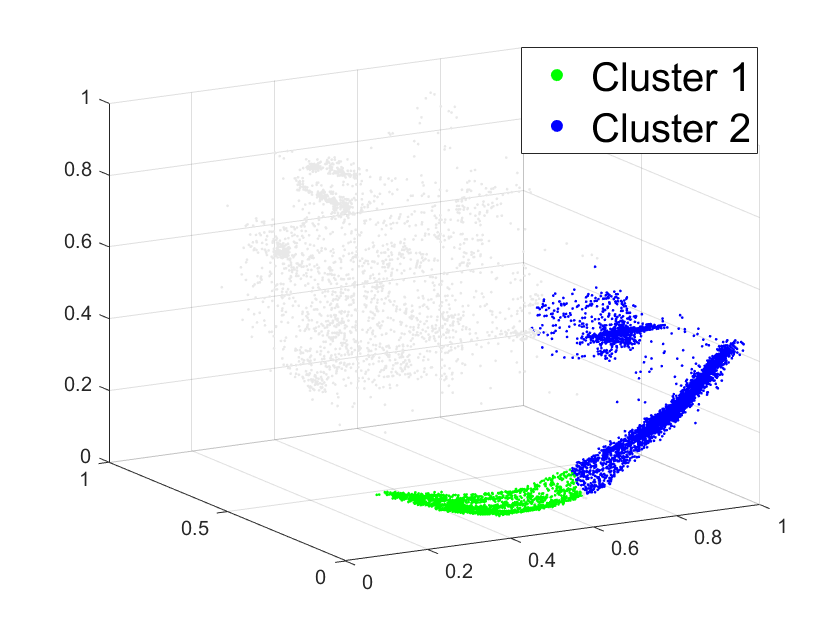} &  
       \includegraphics[width=1.37in]{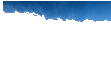} &
      \includegraphics[width=1.37in]{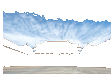} &
       \includegraphics[width=1.37in]{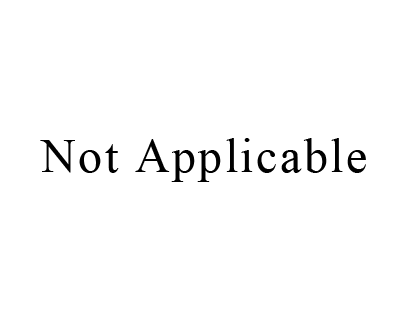} \\      \hdashline
       \begin{turn}{90} \ \ \ \  OPTICS \end{turn}&      
       \includegraphics[width=1.37in]{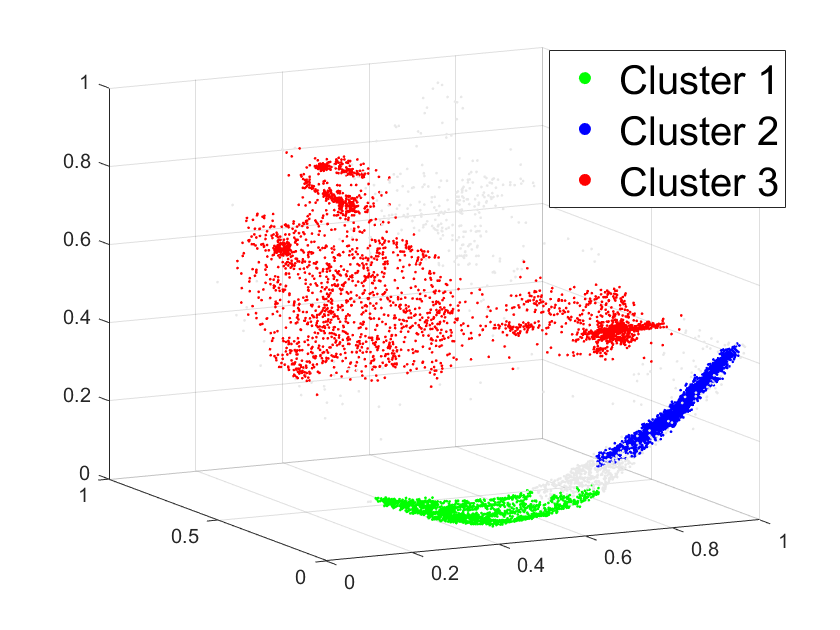} & 
       \includegraphics[width=1.37in]{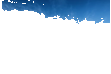} &
      \includegraphics[width=1.37in]{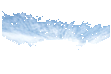} &
      \includegraphics[width=1.37in]{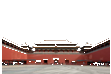}  \\      \hdashline
       \begin{turn}{90} \ \ \ \ \ \ \ SNN \end{turn}&      
       \includegraphics[width=1.37in]{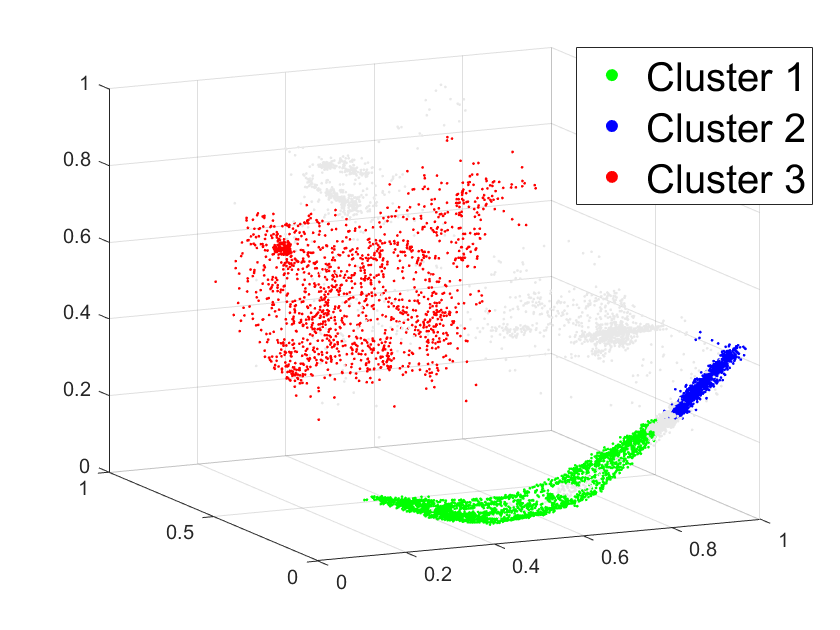} & 
       \includegraphics[width=1.37in]{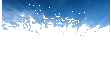} &
      \includegraphics[width=1.37in]{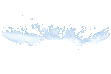} &
      \includegraphics[width=1.37in]{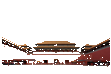}  \\      \hdashline
       \begin{turn}{90} \ \ \  \ \ \ \ \  DP \end{turn}&     
       \includegraphics[width=1.37in]{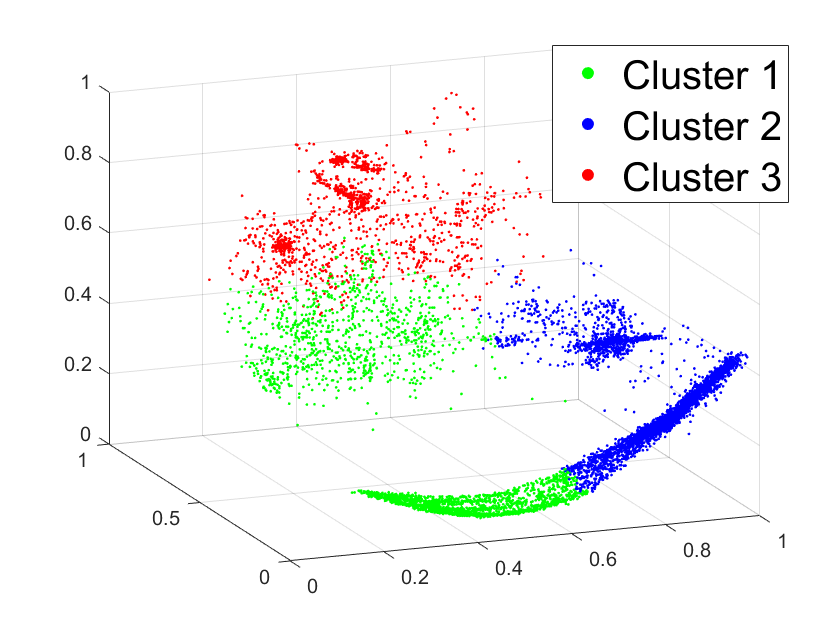} &  
       \includegraphics[width=1.37in]{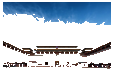} &
      \includegraphics[width=1.37in]{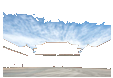} &
      \includegraphics[width=1.37in]{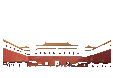}  \\   
      \hline 
  \end{tabular}
\label{seg1}
\end{table} 

Recently, density-ratio has been proposed as a principled  method \cite{ZHU2016983} in addressing density bias. It aims to transform the data such that the estimated density of each transformed point approximates the density-ratio of that point in the original space. This effectively transforms locally low-density gaps between clusters in the original space into globally low-density gaps in the transformed space. As a result, all clusters can be easily extracted using a single density threshold. 

On the same principle, we propose a more effective data transform-and-shift method based on a CDF (Cumulative Density Function). It transforms the space to one in which the cluster gaps are more obvious and all clusters have the same/similar density even though the geometry of each cluster may have changed.   The proposed method is an elegant solution that equalises the density of all clusters (of high-density regions) as well as the density of low-density regions. This matches the implicit assumption of existing algorithms. The proposed method is a pre-processing method such that the transformed-and-shifted data can be used to reduce the underlying bias (towards dense clusters) and boost the performance of existing density-based algorithms without any algorithmic modification. 

This paper makes the following three contributions:

\begin{enumerate}[(a)]  
\item Analysing the limitations and drawbacks of existing density-ratio based methods, i.e., existing CDF transform methods are either individual attribute-based (e.g., ReScale \cite{ZHU2016983}) or individual point-based (e.g., DScale \cite{DSCALE}); thus, they cannot solve the density bias issue adequately. More specifically, ReScale is a one-dimensional method which treats each dimension independently. The independent assumption limits its applications to some datasets only. DScale applies a multi-dimensional CDF without shifting; and it must be incorporated into an existing algorithm. DScale yields a non-metric and asymmetric measure. Both limit their ability to apply the transform multiple times in order to enhance the impact.
\item
Proposing a new CDF-based Transform-and-Shift method called CDF-TS that changes the volume of each point's neighbourhood simultaneously in the full multi-dimensional space. As far as we know, this is the first attempt to perform a multi-dimensional CDF transform to achieve the desired effect of homogenising the distribution of an entire dataset w.r.t a density estimator, i.e., making the density of all clusters homogeneous without impairing cluster structures. 
\item
Applying the new transform-and-shift method to both density-based clustering and distance-based anomaly detection to demonstrate its superiority over four existing algorithms, i.e., DBSCAN \cite{ester1996density}, DP \cite{rodriguez2014clustering},	 LGD \cite{li2019local}, EC \cite{wang2020extreme}; and one $k$NN anomaly detector. 
\end{enumerate} 

The proposed method   CDF-TS differs from the existing density-ratio algorithms of ReScale \cite{ZHU2016983} and DScale \cite{DSCALE}  in three aspects:
\begin{enumerate}[(i)]
\item Methodology: While ReScale \cite{ZHU2016983}, DScale \cite{DSCALE} and the proposed CDF-TS follow the same principled method and have the same aim, the proposed CDF Transform-and-Shift method is a multi-dimensional technique incorporating both transformation and point-shift, which greatly expands the method's applicability. In contrast, ReScale is a one-dimensional transformation and point-shift technique; and DScale is a multi-dimensional transformation without point-shift. The `complete' CDF-TS method produces a more satisfactory result for both clustering and anomaly detection, which we will show in Section \ref{sec_result}. 
\item Metric compliance: CDF-TS and ReScale create a similarity matrix which is a metric; whereas the one produced by DScale \cite{DSCALE} is not a metric, i.e., it does not satisfy the symmetry or triangle inequalities. Thus, DScale cannot be used for algorithms that require the property of symmetry, e.g., we cannot visualise the scaled distance or data distribution using MDS \cite{borg2012applied}.
\item Ease of use: Like ReScale, CDF-TS is a preprocessing method to transform data. The targeted (clustering or anomaly detection) algorithm can thus be applied without altering to the transformed data. In contrast, DScale requires an algorithmic modification before using it.
\end{enumerate}

The rest of the paper is organised as follows. The related work in a broader context is described in Section \ref{related_work}. The most closely related work is organised into the following three sections. Section \ref{issue} presents the issues of inhomogeneous cluster density in density-based clustering and distance-based anomaly detection. Section \ref{denRatio} describes the density-ratio estimation as a principle to address the issues of inhomogeneous density. Section \ref{cdfSca} presents the two existing CDF scaling methods based on density-ratio. Section \ref{sec_CDF-TS} proposes the local-density transform-and-shift as a multi-dimensional CDF scaling method. Section \ref{sec_result} empirically evaluates the performance of existing density-based clustering and distance-based anomaly detection algorithms on the transformed-and-shifted datasets. Discussion and conclusions are provided in  the last two sections.

\section{Related work in a broader context}
\label{related_work}

There are many algorithms proposed to improve density-based clustering performance on a dataset with clusters of varied densities. OPTICS \cite{ankerst1999optics} and HDBSCAN \cite{HDBSCAN2013} utilise different thresholds to extract clusters of different densities. 	  IS-DBSCAN \cite{cassisi2013enhancing}, DP \cite{rodriguez2014clustering}, LC \cite{Chen2018}, LGD \cite{li2019local}, DPC-DBFN \cite{lotfi2020density} and EC \cite{wang2020extreme} rely on local density estimation.  IS-DBSCAN \cite{cassisi2013enhancing} uses a ranking procedure based on the kNN distances to efficiently identify different densities in the dataset. DP \cite{rodriguez2014clustering} identifies cluster modes which have local maximum density and are well separated, and then assigns each remaining point in the dataset to a cluster mode via a linking scheme. LC \cite{Chen2018} and DPC-DBFN \cite{lotfi2020density} utilise enhanced local density estimators based on a kNN graph and a fuzzy neighbourhood measure, respectively, to improve the clustering performance of DP. EC \cite{wang2020extreme} redefines cluster centres as points with the highest density in their neighbourhoods to overcome the shortages of DP. To handle high-dimensional data with varied densities, LGD \cite{li2019local} estimates the local density for each point based on the average distance over the nearest neighbours of that point. In addition, MuDi-Stream \cite{amini2016mudi} uses a grid-based method to detect adaptive micro-clusters for clustering streaming data with clusters of varied densities. MDSC \cite{fahy2019finding}  uses a swarm intelligence approach to adaptively merge micro-clusters and improve the performance of clustering dynamic data streams.

ReScale \cite{ZHU2016983} and DScale \cite{DSCALE} are density-ratio based scaling methods which are able to improve existing density-based clustering algorithm on detecting clusters with varied densities. In addition, there are data-dependent similarity measures, such as SNN \cite{ertoz2003finding} and aNNE \cite{qin2019nearest}, which can be directly applied on an existing density-based clustering algorithm to overcome their weakness of detecting varied densities. 

LOF \cite{LOF2000} is a local density-based approach for anomaly detection. The LOF score of $x$ is the ratio of the average density of $x$'s $k$-nearest neighbours to the density of $x$, where the density of $x$ is inversely proportional to the average distance to its $k$-nearest-neighbours \cite{LOF2000}. It has the ability to identify local anomalies, which are relatively low-density points close to a dense cluster.  With an isolation-based approach to detect local anomalies, iNNE \cite{bandaragoda2018isolation} uses a nearest neighbour ensemble to calculate a ratio of nearest neighbour distances. 

There are many unsupervised distance metric learning algorithms which aim to transform data from a high-dimensional space to a low-dimensional space, e.g., global methods such as PCA \cite{jolliffe2011principal} and KUMMP \cite{wang2011unsupervised}); and local methods such as Isomap \cite{tenenbaum2000global} and t-SNE \cite{maaten2008visualizing}. These methods are usually used as dimension reduction methods such that data clusters are transformed into a low-dimensional space based on some local/global objective function in linear/nonlinear approaches \cite{BLUM1997245,Wang2015}. For example, t-SNE \cite{maaten2008visualizing} first uses the Gaussian kernel to convert similarities between data points to joint probabilities in the original space. Then it produces a low-dimensional space that minimises the Kullback-Leibler divergence between the joint probabilities of the low-dimensional space and the original space. 

The objective of CDF-TS is to reduce the density variation of a given dataset that yields a transformed dataset with clusters of homogeneous density. And this is achieved  without impairing the cluster structures and changing the dimensionality in the given dataset. The transform step employs CDF; the shift step shifts data points based on the local density distribution in the original space. None of the distance metric learning algorithms has this objective, or employs the CDF transform and/or point shift as a methodology. The objective of CDF-TS is specifically targeted to reduce the underlying bias (towards dense clusters) in density-based algorithms. In contrast, metric learning algorithms are generic dimensional reduction methods, often targeted for visualisation. With reference to the categorisation of local/global and linear/nonlinear methods in distance metric learning, CDF-TS uses a linear method to do the CDF transform; and it is a local method because of the use of local density distribution. 

In a nutshell, CDF-TS is neither a metric learning method nor a dimension reduction method.

As mentioned in the last section,  algorithms using a  method closely related to that of CDF-TS are ReScale \cite{ZHU2016983} and  DScale \cite{DSCALE}. With the differences mentioned earlier, we will show that CDF-TS greatly expands the method's applicability and enables different density estimators to be employed in both clustering and anomaly detection tasks.

\section{Effect of inhomogeneous cluster densities}
\label{issue}
 
We now discuss the effect of inhomogeneous cluster densities on density-based clustering and distance-based anomaly detection algorithms.
 
\subsection{Density-based clustering}

This section describes the condition under which density-based algorithm DBSCAN fails to detect all clusters in a dataset of inhomogeneous cluster densities.

Let $ D=\lbrace x_{1}, x_{2}, ..., x_{n} \rbrace$, $x_{i}\in \mathbb{R}^{d}, x_{i} \sim F$ denote a dataset of $n$ points, each sampled independently from a distribution $F$. Let $\widehat{pdf}(x)$ denote the density estimate of point $x$ which approximates the true density $pdf(x)$. In addition, let $\mathcal{N}(x; \epsilon)$ be the $\epsilon$-neighbourhood of $x$, $\mathcal{N}(x; \epsilon)=\lbrace y \in D ~|~ s(x,y) \leqslant \epsilon \rbrace$, where $s(\cdot,\cdot)$ is the distance function ($s:{\mathbb{R}}^{d}\times {\mathbb{R}}^{d}\rightarrow {\mathbb{R}}$).

The classic density-based clustering algorithm DBSCAN \cite{ester1996density}, estimates the density ${pdf}(x)$ of a point $x$ by counting points within a small $\epsilon$-neighbourhood as follows:
\begin{equation}
	\widehat{pdf}_{\epsilon}(x)=\frac{1}{nV_{\epsilon}} \vert  \mathcal{N}(x; \epsilon) \vert=\frac{\vert \lbrace y \in D ~|~ s(x,y) \leqslant \epsilon \rbrace \vert}{nV_{\epsilon}}
	\label{epsN}
\end{equation}

\noindent where $V_{\epsilon} \propto \epsilon^d$ is the volume of a $d$-dimensional ball of radius $\epsilon$.

A set of clusters $\{C_{1},\dots, C_{\varsigma}\}$ is defined as non-empty and non-intersecting subsets: $C_i\subset D, C_i\neq  \emptyset , \forall_{i \ne j} \ C_{i} \cap C_{j}= \emptyset $. Let $c_{i}=\arg\max_{{x\in C_{i}}}\widehat{pdf}(x)$ denote the mode (point of the highest estimated density) for cluster $C_{i}$;  $p_{i}=\widehat{pdf}(c_{i})$ denote the corresponding peak density value. 

DBSCAN uses a global density threshold to identify core points (which have densities higher than the threshold); then it links neighbouring core points together to form clusters \cite{ester1996density}. It is defined as follows.

\begin{definition}
A \emph{core point} is a point with an estimated density above or equal to a user-specified threshold $\tau$, i.e., $(\widehat{pdf}_{\epsilon}(x) \geqslant \tau) \leftrightarrow Core(x)=1$, where $Core$ denotes a set indicator function. 
\end{definition}

\begin{definition}
Using a density estimator with density threshold $\tau$, a point $x_1$ is density connected with another point $x_p$ in a sequence of $p$ unique points from $D$, i.e., $\{x_1,x_2,x_3,...,x_p\}$:  $CON_{\epsilon}^{\tau}(x_1, x_p)$ is defined as: 
 
\begin{equation*}
\small
    CON_{\epsilon}^{\tau}(x_1, x_p) \leftrightarrow  \left\{ \begin{array}{l}
     \text{(i) if $p=2$}:  (x_1 \in \mathcal{N}_{\epsilon}(x_p))  \wedge  (Core(x_1)  \vee Core(x_p)); \\
    \text{(ii) if $p>2$}:  \exists_{(x_1,x_2,...x_{p})} ((\forall_{i\in\{2,...,p\}} \  x_{i-1}  \in  \mathcal{N}_{\epsilon}(x_{i}))  \wedge (\forall_{i\in \{2,...,p-1\}} \  Core(x_i))).
    \end{array} \right. 
\end{equation*}  
\label{def:connect}
\end{definition}

\begin{definition}
A cluster detected by the density-based algorithm DBSCAN is a maximal set of density connected instances, i.e., ${C}_i=\{x\in D \ | \ CON_{\epsilon}^{\tau}(x, c_i)\}$, where $ c_i=\arg\max_{\substack{x\in \mathbb{C}_i}}\widehat{pdf}_{\epsilon}(x)$ is the cluster mode.
\end{definition}

For this kind of algorithm to find all clusters in a dataset, the data distribution must have the following necessary condition: the peak density of each cluster must be greater than the maximum over all possible paths of the minimum density along any path linking any two modes.\footnote{A path linking two modes ${c}_{i}$ and ${c}_{j}$ is defined as a sequence of unique points starting with ${c}_{i}$ and ending with ${c}_{j}$ where adjacent points lie in each other's $\epsilon$-neighbourhood.} This condition is formally described by Zhu et al \cite{ZHU2016983} as follows: 
\begin{equation}
 \min_{{k\in \lbrace1,\dots,\varsigma \rbrace}} c_{k} > \max_{{i\neq j\in \lbrace1,\dots,\varsigma \rbrace}} g_{ij} 
 \label{eqn_condition}
 \end{equation} 

\noindent where $g_{ij}$ is the largest of the minimum density along any path linking the mode of clusters $C_{i}$ and $C_{j}$.

This condition implies that there must exist a threshold $\tau$ that can be used to break all paths between the modes by assigning regions with a density less than $\tau$ to noise. Otherwise, if some cluster mode has a density lower than that of a low-density region between other clusters, then this kind of density-based clustering algorithm will fail to find all clusters. Either some high-density clusters will be merged (when a lower density threshold is used), or low-density clusters will be designated as noise (when a higher density threshold is used). To illustrate, Figure \ref{fig1:a} shows that using a high threshold $\tau_1$ will cause all points in Cluster $C_3$ to be assigned to noise but using a low threshold $\tau_2$ will cause points in $C_1$ and $C_2$ to be assigned to the same cluster.

\begin{figure}
\centering
  \begin{subfigure}[b]{0.346\textwidth}
  \centering\captionsetup{width=.99\linewidth}%
    \includegraphics[width=1\textwidth]{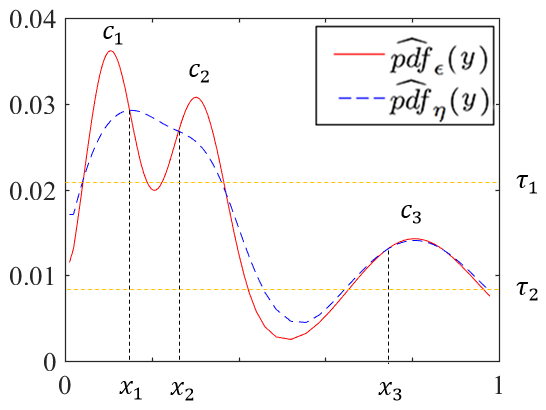}
    \caption{Original data}
    \label{fig1:a}
  \end{subfigure}  %
    \begin{subfigure}[b]{0.346\textwidth}
        \centering\captionsetup{width=.99\linewidth}%
    \includegraphics[width=1\textwidth]{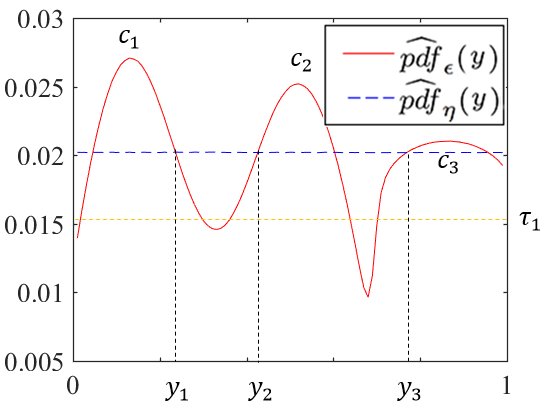}
    \caption{ReScaled data}
        \label{fig1:c}
  \end{subfigure}
   \caption{(a) A mixture of three one-dimensional Gaussian distributions that cannot be separated using a single density threshold; (b) Density distribution on ReScaled data of (a), where a single density threshold can be found to separate all three clusters. Note that point $x_1$, $x_2$ and $x_3$ are shifted to $y_1$, $y_2$ and $y_3$, respectively. $\eta$ is a larger bandwidth than $\epsilon$.}
    \label{fig1} 
\end{figure}

\subsection{$k$NN anomaly detection}

A classic nearest-neighbour based anomaly detection algorithm assigns an anomaly score to an instance based on its distance to the $k$-{th} nearest neighbour \cite{Ramaswamy:2000}. The instances with the largest anomaly scores are identified as anomalies.

On a dataset of inhomogeneous cluster densities, we show that this kind of anomaly detector fails to detect two types of anomalies, i.e., local anomalies and anomalous clusters.

Given a dataset and the parameter $k$, the density of $x$ can be estimated using a $k$-th nearest neighbour density estimator (as used by the classic $k$NN anomaly detector \cite{Ramaswamy:2000}):
\begin{equation}
  \widehat{pdf}_{\epsilon_k(x)}(x) = \frac{1}{nV_{\epsilon_k(x)}} \vert  \mathcal{N}(x;\epsilon_k(x))  \vert = \frac{k}{n \times V_{\epsilon_k(x)}} \propto (\frac{1}{\epsilon_k(x)})^d
   	\label{kdist}
 \end{equation}

 \noindent
where $\epsilon_k(x)$ is the distance between $x \in \mathbb{R}^d$ and its $k$-th nearest neighbour in a dataset $D$.

Note that the $k$-th nearest neighbour distance $\epsilon_k(x)$ is a proxy to the density of $x$, i.e., high $\epsilon_k(x)$ indicates low density, and vice versa.

Let $C$ be the set of all normal points in a dataset $D$. The condition under which the classic $k$NN anomaly detector \emph{could}, with an appropriate setting of a density/distance threshold, identify every anomaly $y$ in $A=D\setminus C$  is given as follows:
\begin{equation}
\min_{y\in A} \ \epsilon_k(y) > \max_{x\in C} \ \epsilon_k(x) 
 \label{knn}
\end{equation}

Equation \ref{knn} states that all anomalies must have the highest $k$NN distances (or the lowest densities) in order to detect them. In other words, $k$NN anomaly detectors can detect both global anomalies and scattered anomalies which have lower densities than that of all normal points \cite{aggarwal2016outlier,Liu2010}.

However, based on this characteristic, $k$NN anomaly detectors are unable to detect:
 \begin{enumerate}[(a)]
\item Local anomalies. This is because a local anomaly $y$ with low density relative to nearby normal (non-anomalous) instances in a region of high average density may still have a higher density than that of normal (non-anomalous) instances in regions of lower average density \cite{LOF2000}. Translating this in terms of  $k$-th NN distance, we have: $\forall_{x,z\in C, y\in A}\  \epsilon_k(x) < \epsilon_k(y) < \epsilon_k(z)$.  Here, some local anomalies (for example, points located around the boundaries of $C_1$ and $C_2$) are ranked lower than the normal points located around the sparse cluster ($C_3$), as shown in Figure \ref{fig1:a}. 
\item  
Anomalous clusters (sometimes referred to as clustered anomalies \cite{Liu2010}) are groups of points where each group is too small to be considered a normal cluster and is found in a low-density region of the space (i.e. far from the other normal clusters). They are difficult to detect when the density of each point of an anomalous cluster is higher than most normal points. A purported benefit of the $k$NN anomaly detector is that it is able to identify such anomalous clusters provided $k$ is sufficiently large (larger than the size of the anomalous group), when these clustered anomalies are sufficiently distanced from normal points \cite{aggarwal2016outlier}. By rescaling the data such that clustered anomalies could be farther to the normal points, larger $k$NN distances could be used to detect them. 
 \end{enumerate} 

\begin{figure}
\centering
  \begin{subfigure}[b]{0.346\textwidth}
  \centering\captionsetup{width=.99\linewidth}%
    \includegraphics[width=1\textwidth]{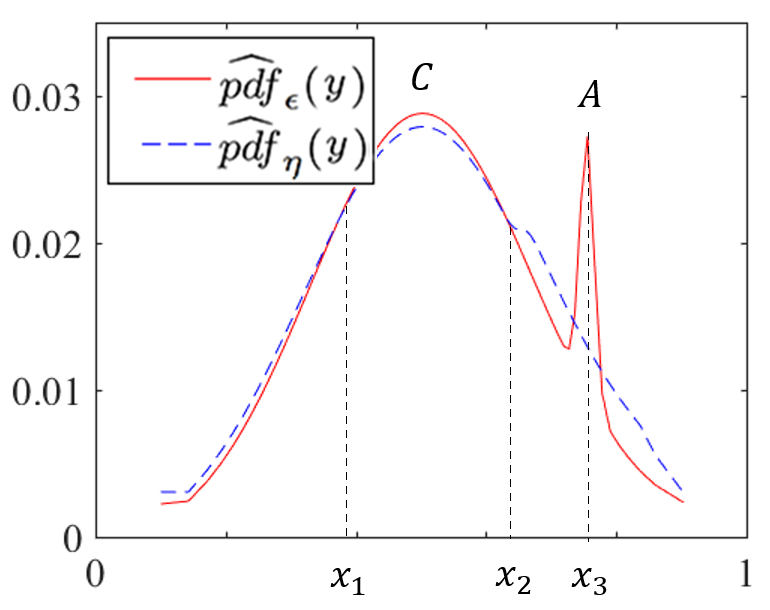}
    \caption{Original data}
    \label{fig2:a}
  \end{subfigure}  %
 \begin{subfigure}[b]{0.346\textwidth}
        \centering\captionsetup{width=.99\linewidth}%
    \includegraphics[width=1\textwidth]{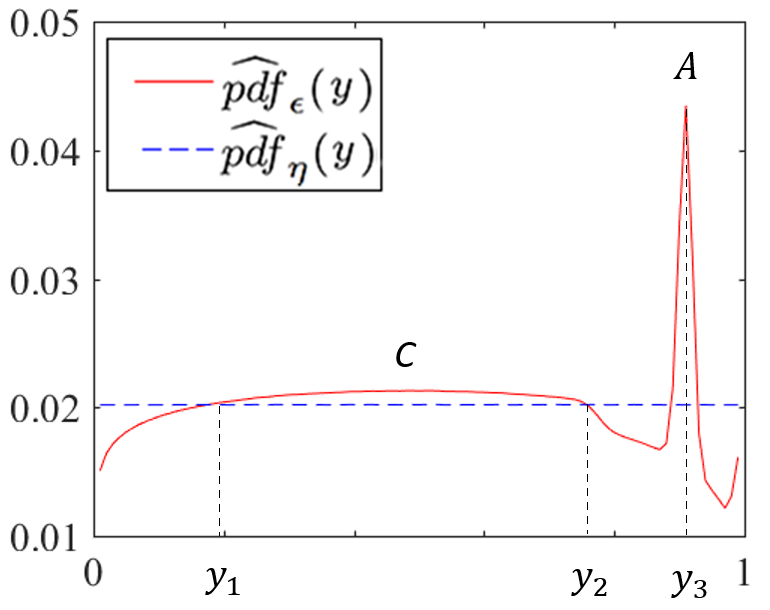}
    \caption{ReScaled data}
        \label{fig2:c}
  \end{subfigure}
   \caption{(a) A mixture of two one-dimensional Gaussian distributions where C is a normal cluster and A is an anomalous cluster; (b) Density distribution on the ReScaled data of (a), where the anomalous cluster is farther to the normal cluster centre. Note that point $x_1$, $x_2$ and $x_3$ are shifted to $y_1$, $y_2$ and $y_3$, respectively.}
    \label{fig2} 
\end{figure}

 \hspace*{1cm} 

To sum up, both density-based clustering and distance-based anomaly detectors have weaknesses when it comes to handling datasets with inhomogeneous cluster densities clusters. Rather than creating a new density estimator or modifying the existing clustering and anomaly detection algorithm procedures, we advocate transforming clusters to be more uniformly distributed (with homogeneous density) than it is in the original space such that the separation between clusters can be identified easily.

\section{Density-ratio estimation}
\label{denRatio}

Density-ratio estimation is a principled method to overcome the weakness of density-based clustering for detecting clusters with inhomogeneous densities \cite{ZHU2016983}.

The density-ratio of a point is the ratio of two density estimates calculated using the same density estimator, but with two different bandwidth settings.  

Let $pdf(\cdot; \gamma)$ and $pdf(\cdot;\lambda)$ be density estimators using kernels of bandwidth $\gamma$ and $\lambda$, respectively. Given the constraint that the denominator has a larger bandwidth than the numerator $\gamma <  \lambda$, the density ratio of $x$ is estimated as:
  \begin{equation}
   {rpdf}(x;\gamma,\lambda)  =\frac{{pdf}(x;\gamma)}{{pdf}(x;\lambda)} 
    \label{equationRpdf1}
   \end{equation}
 
We recall a lemma from \cite{ZHU2016983} regarding the density ratio:
\begin{lemma}
For any data distribution and sufficiently small values of $\gamma$ and $\lambda$ s.t. $\gamma < \lambda$, if $x$ is at a local maximum density of $\mathcal{N}(x;\lambda)$, then ${rpdf}(x;\gamma,\lambda)\geqslant 1 $;  and if $x$ is at a local minimum density of $\mathcal{N}(x;\lambda)$, then ${rpdf}(x;\gamma,\lambda)\leqslant 1$.
\label{ratioThre}
\end{lemma}

Since points located at local high-density areas (almost invariably) have density-ratio  higher than  points located at local low-density areas, a global density-ratio threshold around unity can be used to identify all cluster peaks and break all paths between different clusters. Thus, based on density-ratio estimation, existing density-based clustering algorithms such as DBSCAN can identify clusters as regions of locally high density, separated by regions of locally low density.

Similarly, a density-based anomaly detector is able to detect local anomalies since their density-ratio values are lower and ranked higher than normal points with locally high densities.

\section{Scaling inhomogeneous cluster densities}
\label{cdfSca}
We now discuss methods for rescaling a dataset to tackle the problem of inhomogeneous cluster densities.

\subsection{One-dimensional CDF scaling}

Let $pdf(\cdot; \lambda)$ and $cdf(\cdot; \lambda)$ denote density and cumulative-density estimators, respectively, which are parameterised by a bandwidth of $\lambda$. Let $x'$ denote a point that has been transformed using the $cdf$ as follows:
\begin{equation}
  x' = cdf(x;\lambda) 
\end{equation}
 
As this is a probability integral transform \cite{pearson1938probability}, we then have the property:
\begin{equation} 
pdf(x';\lambda)  = pdf(cdf(x;\lambda);\lambda)= 1/n
\end{equation}

Given the data size $n$ is large, $pdf(x)$ varies slowly around $x$ and $\lambda > \gamma$, we have  \cite{ZHU2016983}: 


\begin{eqnarray}
pdf(x';\gamma) & = & \frac{1}{nV_{\gamma}}\sum_{j}\mathbf{1}(\Vert{x'-x'_j}\Vert \leqslant \gamma)  \nonumber\\
 & = &
\frac{1}{nV_{\gamma}}\sum_{j}\mathbf{1}(\Vert{cdf(x;\lambda)-cdf(x_j;\lambda)}\Vert \leqslant \gamma) \nonumber  \\
 & \approx & 
\frac{1}{nV_{\gamma}}\sum_{j}\mathbf{1}(pdf(x;\lambda)*\Vert{x-x_j}\Vert \leqslant \gamma)   \nonumber \\
& \approx &
\frac{1}{V_{\gamma}} P_{X\sim F}(pdf(x;\lambda)*\Vert{x-X}\Vert \leqslant \gamma) \nonumber \\
& = &
\frac{1}{V_{\gamma}} P_{X\sim F}(\Vert{x-X}\Vert \leqslant \frac{\gamma}{pdf(x;\lambda)}) \nonumber \\
  & \approx &
\frac{1}{V_{\gamma}} \frac{P_{X\sim F}(\Vert{x-X}\Vert \leqslant \gamma)}{pdf(x;\lambda)}     \nonumber \\ 
&\approx  &
\frac{1}{nV_{\gamma}} \frac{\sum_{j}\mathbf{1}(\Vert{x-x_j}\Vert \leqslant \gamma)}{pdf(x;\lambda)}   =  \frac{pdf(x;\gamma)}{pdf(x;\lambda)}  
 \end{eqnarray}

\noindent where  $\mathbf{1}(.)$ denotes the indicator function and $\Vert.\Vert$ is the Euclidean distance function. 

ReScale \cite{ZHU2016983} is a representative implementation algorithm based on this $cdf$ transform. Figure \ref{fig1:c} and Figure \ref{fig2:c} show ReScale rescales the data distribution on two 1-dimensional datasets, respectively. They show that clusters and anomalies are easier to identify after the application of ReScale. 

Since this $cdf$ transform can be performed on a one-dimensional dataset only, ReScale must apply the transformation to each attribute independently for a multi-dimensional dataset. 

\subsection{Multi-dimensional CDF scaling}
\label{sec_MD_CDF}

With a distance scaling method, a multi-dimensional $cdf$ transform can be achieved by simply rescaling the distances between each point and all the other instances in its local neighbourhood, (thereby considering all of the dimensions of the multi-dimensional dataset at once).

Given a point $x\in D$, the distance between $x$ and all points in its $\lambda$-neighbourhood $y\in \mathcal{N}(x;s,\lambda)$\footnote{For reasons that will become obvious shortly, we now reparameterise the neighbourhood function $\mathcal{N}(x;s,\lambda)$ to depend on the distance measure $s$.} can be rescaled using a scaling function $r(\cdot)$: 

\begin{equation}
\small
s'(x,y)=s(x,y)\times r(x;\lambda), \forall_{x,y\in D, y\in \mathcal{N}(x;s,\lambda)} 
\label{Scaling1}
\end{equation}

\noindent where  $s'(\cdot,\cdot)$ is the scaled distance of $s(\cdot,\cdot)$.

Here, the scaling function $r(x;\lambda)$ depends on both the position $x$ and size of the neighbourhood $\lambda$. It is defined as follows using the estimated density $pdf(x;s,\lambda)$ with the aim of making the density distribution within the $\lambda$-neighbourhood uniform:
\begin{eqnarray}
\small
r(x;\lambda) = \frac{m}{\lambda}\times (\frac{pdf(x;s,\lambda) \times V_{\lambda}}{n})^{\frac{1}{d}} \propto pdf(x;s,\lambda)^{\frac{1}{d}}
\label{rate1}
\end{eqnarray}
\noindent where $m = max_{x,y \in D} \ s(x,y)$ is the maximum pairwise distance in $D$.  Note that we have now reparameterised the density estimator $pdf(x;s,\lambda)$ to include the distance function s, in order to facilitate calculations below. 

With reference to the uniform distribution which has density $\frac{n}{V_m}$, where $V_m$ is the volume of the ball having radius $m$:
\begin{eqnarray}
\small
r(x;\lambda) < 1 \mbox{, if } pdf(x;s,\lambda) < \frac{n}{V_m} \\
r(x;\lambda) > 1 \mbox{, if } pdf(x;s,\lambda) > \frac{n}{V_m}
\label{rate2}
\end{eqnarray}
That is, the process rescales sparse regions to be denser by using $r(x;\lambda)<1$, and dense regions to be sparser by using $r(x;\lambda)>1$;  such that the entire dataset is approximately uniformly distributed in the scaled $\lambda$-neighbourhood. More specifically, after rescaling distances with s', the density of points in the neighbourhood of size ${\lambda_x'}={\lambda}\times r(x;\lambda)$ around $x$ is the same as the density of points across the whole dataset:

\begin{eqnarray}
\small
	{pdf} (x;s',\lambda_x') =\frac{pdf(x;s,\lambda)\times V_{\lambda}}{V_{\lambda_x'}}=\frac{pdf(x;s,\lambda)\times V_{\lambda}}{V_{\lambda}\times r(x;\lambda)^d}  \nonumber \\  =\frac{pdf(x;s,\lambda)\times V_{\lambda}}{V_{\lambda}\times (\frac{m}{\lambda}\times (\frac{pdf(x;s,\lambda) \times V_{\lambda}}{n})^{\frac{1}{d}} )^d} = \frac{n\times\lambda^d}{V_\lambda\times m^d} = \frac{n}{V_{m}}
\label{eq1}
\end{eqnarray}

\noindent Note that the above derivation is possible because (i) the scaling is isotropic about $x$ (hence the shape of the unit ball doesn't change only its size) and (ii) the uniform-kernel density estimator is local (i.e., its value depends only on points within the $\lambda_x'$ neighbourhood).  

In order to maintain the same relative ordering of distances between $x$ and all other points in the dataset, the distance between $x$ and any point outside the $\lambda$-neighbourhood $y \in D\setminus\mathcal{N}(x;s,\lambda)$  can be normalised by a simple $min$-$max$ normalisation:
\begin{eqnarray}
\small
 s'({x,y})=(s({x,y})-\lambda)\times \frac{m-\lambda_x'}{m-\lambda}+\lambda_x', \forall_{y \in D\setminus\mathcal{N}(x;s,\lambda)} 
\label{Scaling2}
\end{eqnarray}

The implementation of DScale \cite{DSCALE}, which is a representative algorithm based on this multi-dimensional CDF scaling, is provided in  \ref{appA}.

An observation about the density on the rescaled distance is given as follows:

\begin{observation}
\label{ob1}
The density ${pdf}(x;s', \gamma_x')$ with the rescaled distance $s'$ is approximately proportional to the density-ratio $\frac{pdf(x;s, \gamma)}{pdf(x;s, \lambda)}$ in terms of the original distance $s$ within the $\lambda$-neighbourhood of $x$.
\end{observation}

\begin{proof} 

\begin{eqnarray}
\small
{pdf}(x;s',\gamma_x')  \approx   \frac{pdf(x;s,\gamma)\times V_{\gamma}}{V_{\gamma_x'}} \nonumber
	 = 
\frac{pdf(x;s,{\gamma})\times V_{\gamma}}{V_{\gamma}\times r(x;\lambda)^d}
\nonumber  \\\nonumber	 =  \frac{pdf(x;s,{\gamma})}{\frac{m^d}{\lambda^d}\times \frac{pdf(x;s,{\lambda})\times V_{\lambda}}{n}} 	 = 
\frac{n}{V_m} \times \frac{pdf(x;s,{\gamma})}{pdf(x;s,{\lambda})} \nonumber  \\  \propto  \frac{pdf(x;s,{\gamma})}{pdf(x;s,{\lambda})} \nonumber
\end{eqnarray}

\noindent where ${\gamma_x'}={\gamma}\times r(x;\lambda)$ and $\gamma<\lambda$.  
\end{proof} 

Note that the above observation is only valid within the $\lambda$-neighbourhood of $x$, and when each $x\in D$ is treated independently. In other words, the $cdf$ transform is only valid locally.
In addition, the rescaled distance is asymmetric, i.e.,  $s'(x,y) \ne s'(y,x)$ when  $r(x;\lambda) \ne r(y;\lambda)$.

To be a valid $cdf$ transform globally for the entire dataset, we propose in this paper to perform an iterative process that involves two steps: distance rescaling and point shifting.

\section{CDF Transform-and-Shift}
\label{sec_CDF-TS}

Here we present our proposed method, which we call CDF Transform-and-Shift (CDF-TS) because the key process is motivated by and based on a multi-dimensional CDF scaling.

A CDF transform on a dataset produces a uniformly distributed dataset \cite{shi2000reducing}. However, an entire dataset of a single uniform distribution destroys the cluster structures in the dataset which is of no use for data mining.

The proposed method aims to reduce the density variation of a given dataset, effectively equalises the density of all clusters (of high-density regions) as well as equalises the density of low-density regions (between any two clusters). The latter ensures that the cluster structure in the dataset is preserved; while the former conforms to the implicit assumption of existing density-based algorithms. These enable each of the existing algorithms to identify all clusters that would be otherwise impossible.

This is achieved through a CDF transform-and-shift process. While the same CDF transform process as DScale is used, the additional ``shift'' ensures that (i) the transformed-and-shifted dataset becomes approximately uniformly distributed in the scaled $\lambda$-neighbourhood; and (ii) the standard Euclidean distance can then be used to measure distances between any two transformed-and-shifted points. These advantages are not available in DScale which relies on a rescaled distance which is non-metric and asymmetric.

The CDF transform-and-shift process is described as follows.

Consider two points $x,y \in D$. In order to make the distribution around point $x$ more uniform, we wish to rescale (expand or contract) the distance between $x$ and $y$ to be $s'(x,y)$ as defined by Equations \ref{Scaling1} and \ref{Scaling2}. We do this by translating $y$ to a new point denoted $y_x'$ which lies along the direction $(y-x)$ as follows:
\begin{equation}
y_x' = x + \frac{s'(x,y)}{s(x,y)}(y-x)
\label{e15}
\end{equation}
Note that the distance between $x$ and the newly transformed point $y_x'$ satisfies the rescaled distance requirement: 
\begin{equation}
s(x,y_x') = ||y_x'-x|| = \frac{s'(x,y)}{s(x,y)}||(y-x)|| = s'(x,y)
\end{equation}

So far we have considered the effect of translating point $y$ to $y_x'$ wtih respect to $x$ only. 
To consider the effects with respect to \emph{all} points $D$,
point $y$ needs to be transformed by the average of the above translations with respect to all points in $D$: 
\begin{equation}
\bar{y}'=\frac{1}{n}\sum_{x\in D} y_x' = \frac{1}{n}\sum_{x\in D} \left[ x + \frac{s'(x,y)}{s(x,y)}(y-x) \right]
\label{e17}
\end{equation}

We make the following observation regarding the final shifted dataset, denoted $D'=\{\bar{y}'|y\in D\}$:
\begin{observation} 
Given a dataset, if there is a sufficiently small bandwidth $\lambda\!<\!m$ such that for every point $x\!\in\!D$ the density over its neighbourhood $\mathcal{N}(x;s,\lambda)$ varies relatively slowly,
then the variance in $\lambda$-neighbourhood density should reduce as a result of the data transformation, i.e.:
\begin{equation}
\mathbb{VAR}_{\bar{x}'\in D'}[pdf(\bar{x}';s,\lambda)] < \mathbb{VAR}_{x\in D}[pdf(x;s,\lambda)] 
\end{equation}
We provide an intuitive argument supporting this observation in  \ref{appB}.\footnote{It is not clear whether this condition (reduction in $pdf$ variance for small $\lambda$) will invariably hold for all possible datasets. We leave the investigation as to whether the property holds for all possible data configurations to future work.} 
\label{Obs2}
\end{observation}
 
Furthermore, the transform-and-shift process can be repeated multiple times iteratively to further reduce the variance in the $\lambda$-neighbourhood density estimates.

Although CDF-TS takes the average of scaling effects from all data points, the $\lambda$-neighbourhood densities of all points approximate to be uniformly distributed. Since Equations \ref{e15} and \ref{e17} are based on the linear transformation within and outside the $\lambda$-neighbourhood of a data point, the local density-variation inside the $\lambda$-neighbourhood can be preserved in each shift process, i.e., the locally low-density gaps between clusters become globally low-density gap, as shown in Observation \ref{ob1}.

The key parameter $\lambda$ is very important for this shift process. In case a non-convergence situation arises from an inappropriate $\lambda$ setting, we set a condition where the iteration process stops: when a fixed iteration limit is reached, or when the total movement of points between iterations ($D$ to $D'$) falls below a threshold $\delta$. Here we use a Manhattan distance to measure the movement in our experiments, i.e.,
\begin{equation} 
\sum_{x \in D,\ \bar{x}' \in D'}{\vert x- \bar{x}' \vert} \le \delta
\label{obj}
\end{equation} 
\noindent

Figure \ref{hard} illustrates the effects of CDF-TS on a two-dimensional dataset with different iterations. It shows that the original clusters with different densities become increasing uniform with less density variation as the number of iterations increases. In addition, the gaps between different clusters become more obvious after iterations.

\begin{figure*}[!tb]
\centering
  \begin{subfigure}[b]{0.24\textwidth}
 \centering\captionsetup{width=.9\linewidth}%
    \includegraphics[width=1.37in]{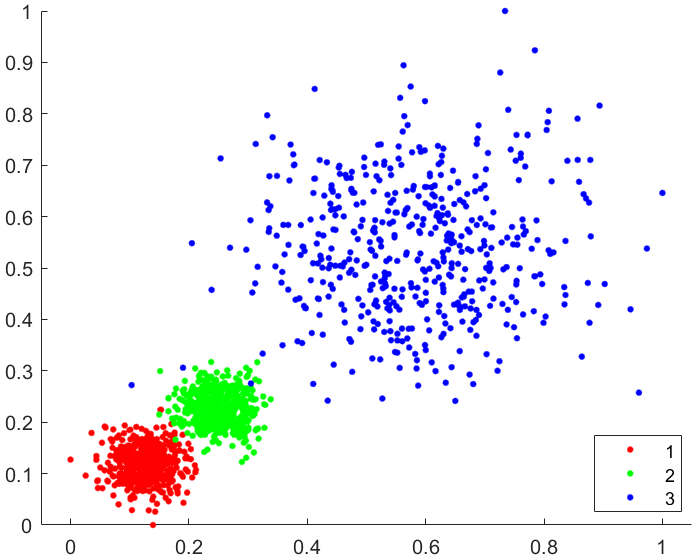}
    \caption{A two-dimensional data with $Std(x)=4.29$}
    \label{hard:a}
  \end{subfigure}  %
  \begin{subfigure}[b]{0.24\textwidth}
 \centering\captionsetup{width=.9\linewidth}%
    \includegraphics[width=1.37in]{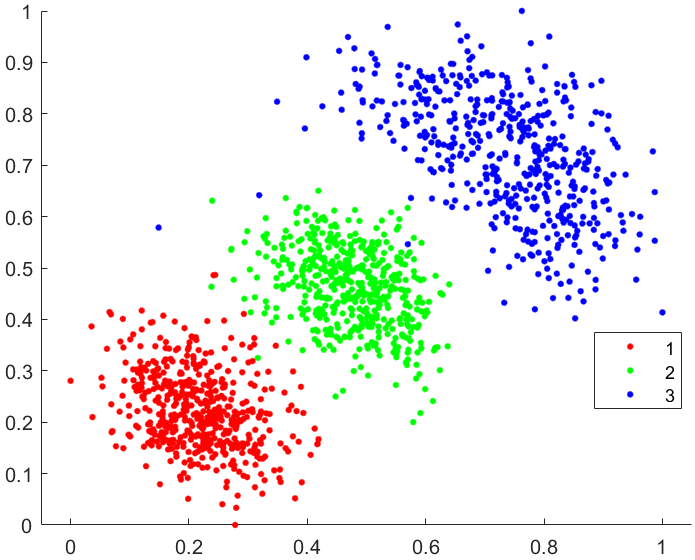}
    \caption{After 1 iteration $Std(x)=1.65$}
        \label{data:n}
  \end{subfigure}
  \begin{subfigure}[b]{0.24\textwidth}
 \centering\captionsetup{width=.9\linewidth}%
    \includegraphics[width=1.37in]{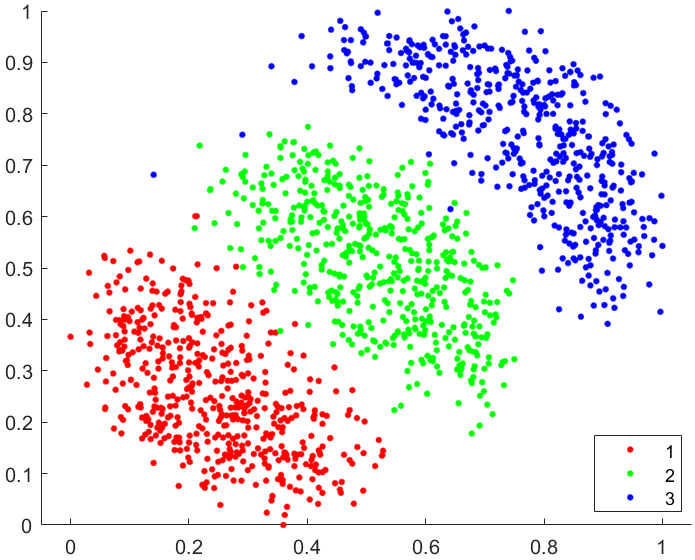}
    \caption{After 3 iterations $Std(x)=0.62$}
    \label{hard:c}
  \end{subfigure}  %
  \begin{subfigure}[b]{0.24\textwidth}
 \centering\captionsetup{width=.9\linewidth}%
    \includegraphics[width=1.37in]{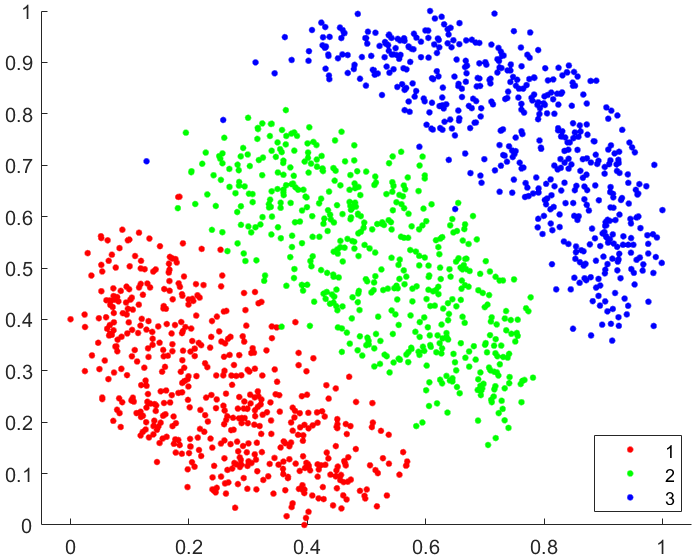}
    \caption{After 5 iterations $Std(x)=0.42$}
        \label{hard:d}
  \end{subfigure}
\caption{Scatter plots for illustrating the effects of CDF-TS on a two-dimensional data with $\lambda=0.1$. $Std(x)=\sigma(	\widehat{pdf}_{\epsilon}(x\in D))$ represents the standard deviation of the density with $\epsilon=0.1$.} 
    \label{hard} 
\end{figure*}

The implementation of CDF-TS  is shown in Algorithm~\ref{FScaling}. The DScale algorithm used in step 6 is provided in  \ref{appA}. Algorithm~\ref{FScaling} requires two parameters $\lambda$ and $\delta$.  
Note that, due to the shifts in each iteration, the value ranges of attributes of the shifted dataset $D'$ are likely to be different from those of the original dataset $D$. We can use a min-max normalisation \cite{aksoy2001feature} on each attribute to keep the values in the same fixed range at the end of each iteration. In addition, we set the maximum number of iterations to prevent the unconverged situation. 

\begin{algorithm}[!htb]
\small
	\caption{CDF-TS($D$, $\lambda$, $\delta$)}
	\begin{algorithmic}[1] 
		\Require $D$ - input data matrix ($n  \times d$ matrix); $\lambda$ - bandwidth parameter; $\delta$ - threshold for the total shifted distance.
		\Ensure $D'$ - data matrix after rescaling and point shifting. 
		\State Normalising $D$ using min-max normalisation
		\State $\Delta=\infty$
		\State $t=1$
		\While {$\Delta>\delta$ and $t<100$}      
    	\State $S \leftarrow$ Calculating the distance matrix for $D$	
    	\State $S' \leftarrow$ DScale($S$, $\lambda$, $d$)  
        \For {each $z \in D$ (where $z$ is a reference point)}      
		\State  $D_z \leftarrow$ Shift every point $x\in D$  to $x_z'$ in the direction of $z$ to $x$ with magnitude $S'[z,x]$
		\EndFor   
		\State $D'\leftarrow \frac{1}{n}\sum_{z \in D} D_z$	
        \State Normalising $D'$ using min-max normalisation
 		\State $\Delta=\frac{1}{nd}\sum_{i,j}\vert D[i,j]-D'[i,j]\vert$
 		\State $t=t+1$
		\State $D=D'$ 
		\EndWhile             
		\State \Return $D'$
	\end{algorithmic}
	\label{FScaling} 
\end{algorithm}

To demonstrate the effects of different density-ratio based rescaling methods (ReScale \cite{ZHU2016983}, DScale \cite{DSCALE} and CDF-TS), we apply them for the same image segmentation task as shown in Table \ref{seg1}. Table \ref{seg2} shows the results of applying the clustering algorithm DP \cite{rodriguez2014clustering} to produce three clusters on the image shown in Figure~\ref{lab:aa}. The scatter plots in LAB space in Table \ref{seg2} show that the three clusters become more homogeneous distributed. Although the density variation of CDF-TS is not the smallest, the gaps between cluster boundaries are much larger for CDF-TS than the other scaling methods. Thus, DP with CDF-TS yields the best clustering result, shown in Table \ref{seg2}. This also shows that density equalisation is not the only aim; and by itself, density equalisation does not achieve the intended outcome of identifying all clusters in the dataset.

\begin{table} [!htb] 
  \renewcommand{\arraystretch}{1.2}
 \setlength{\tabcolsep}{3.5pt}
\centering
\caption{DP's image segmentation on the image shown in Figure~\ref{lab:aa}. The colours in the scatter plots indicate the three clusters identified by DP using each of the four methods. The scatter plots of DScale and No Scaling are based on the original LAB attributes; and the other two are based on the transformed attributes. Note that the standard deviation of the density distribution of points in the LAB space are 12.82, 7.03, 3.19 and 7.85 for the four sub-figures, respectively. The density estimation is based on Equation \ref{epsN} with $\epsilon=0.1$.} 
  \begin{tabular}{|c|c|ccc|}
    \hline
     & LAB & Cluster 1  & Cluster 2   & Cluster 3  \\
      \hline
       \begin{turn}{90}  \ \ \   No Scaling \end{turn}&      \includegraphics[width=1.37in]{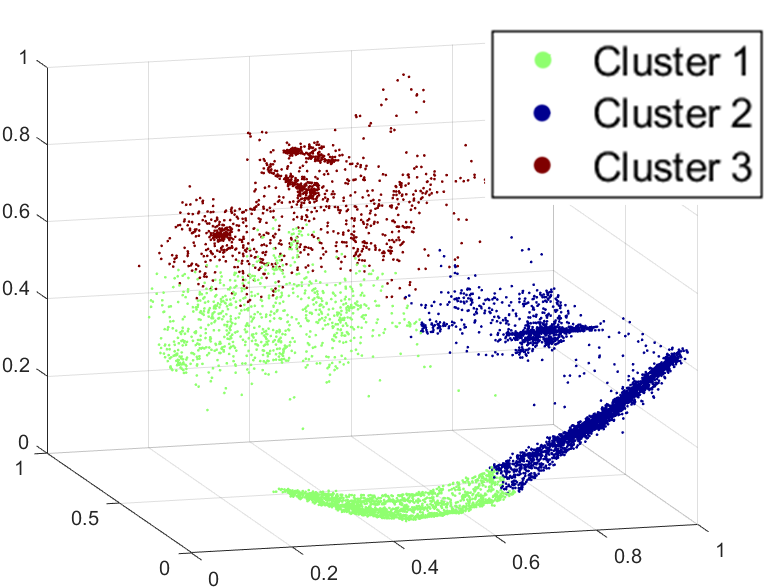} &  \includegraphics[width=1.37in]{pic/DIS22.png} &
      \includegraphics[width=1.37in]{pic/DIS21.png} &
      \includegraphics[width=1.37in]{pic/DIS23.png}  \\      \hdashline
       \begin{turn}{90} \ \ \ \  ReScale \end{turn}&      \includegraphics[width=1.37in]{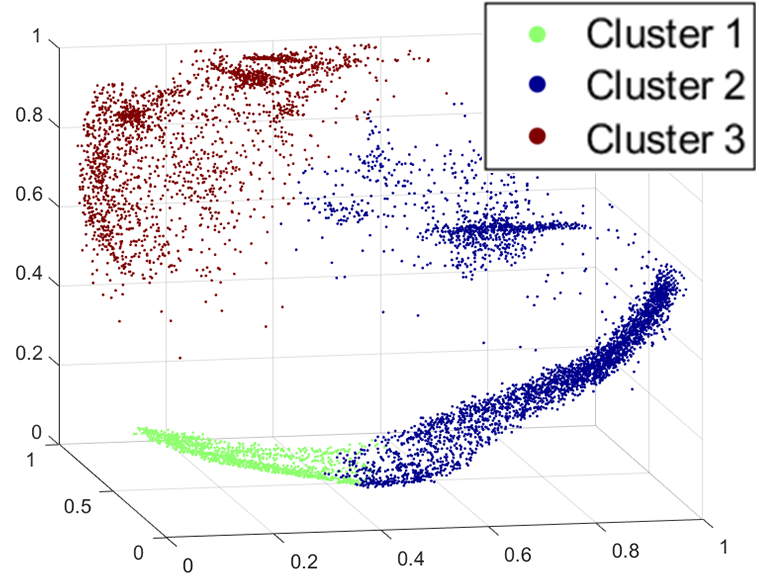} & \includegraphics[width=1.37in]{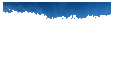} &
      \includegraphics[width=1.37in]{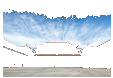} &
      \includegraphics[width=1.37in]{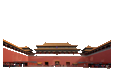}  \\      \hdashline
       \begin{turn}{90} \ \ \ \ \ \ DScale \end{turn}&      \includegraphics[width=1.37in]{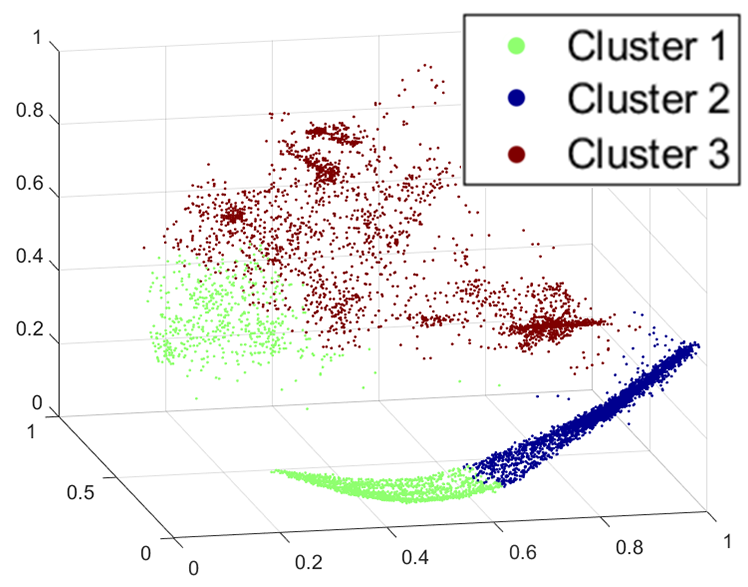} & \includegraphics[width=1.37in]{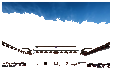} &
      \includegraphics[width=1.37in]{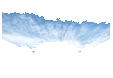} &
      \includegraphics[width=1.37in]{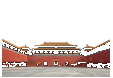}  \\      \hdashline
       \begin{turn}{90} \ \ \ \ \ \  CDF-TS \end{turn}&      \includegraphics[width=1.37in]{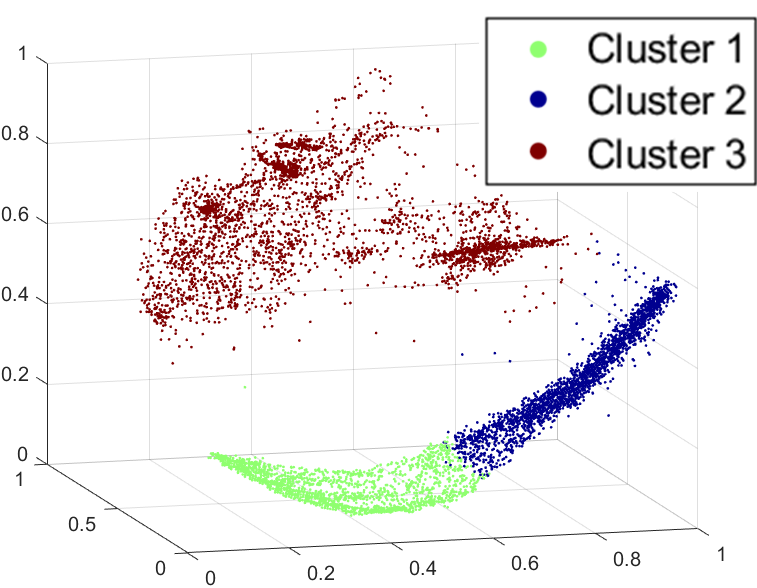} & \includegraphics[width=1.37in]{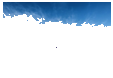} &
      \includegraphics[width=1.37in]{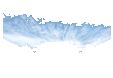} &
      \includegraphics[width=1.37in]{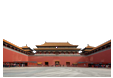}  \\ 
      \hline 
  \end{tabular}
\label{seg2}
\end{table} 
 
\subsection{Density estimators applicable to CDF-TS}

The following two types of density estimators, with fixed-size and variable-size bandwidths of the uniform kernel, are applicable to CDF-TS to do the CDF scaling:

\noindent
(a) {\bf $\epsilon$-neighbourhood density estimator}:
\[
\widehat{pdf}_{\epsilon}(x)=\frac{\vert \lbrace y \in D ~|~ s(x,y) \leqslant \epsilon \rbrace \vert}{nV_{\epsilon}}
\]
\noindent where $V_{\epsilon}$ is the volume of the ball with radius $\epsilon$.

When this density estimator is used in CDF-TS, the $\lambda$-neighbourhood described in Sections \ref{sec_MD_CDF} and \ref{sec_CDF-TS}, i.e., $\mathcal{N}(x;s,\lambda)= \lbrace y \in D ~|~ s(x,y) \leqslant \epsilon \rbrace$, where $\lambda=\epsilon$, denoting the fixed-size bandwidth uniform kernel.

\vspace{3mm}
\noindent
(b) {\bf $k$-th nearest neighbour density estimator}:
\[
\widehat{pdf}(x;\epsilon_k(x))= \frac{k}{nV_{s(x,x_k)}} \propto (\frac{1}{s(x,x_k)})^d
\]
\noindent where $x_k$ is the $k$-th nearest neighbour of $x$; $\epsilon_k(x)$ is $k$-th nearest neighbour distance of $x$.

When this density estimator is used in CDF-TS, the neighbourhood size $\lambda$ becomes dependent on the location $x$ and the distribution of surrounding points. We denote $\lambda(x)_k = s(x,x_k)$ as the distance to the $k$-th nearest neighbour of $x$. In this case the density is simply calculated over the neighbourhood: $\mathcal{N}(x;s,\lambda(x)_k) =\lbrace y \in D ~|~ s(x,y) \leqslant s(x,x_k) \rbrace$, denoting the variable-size bandwidth uniform kernel, i.e., small in dense region and large in sparse region. Note that, in this circumstance, the density of each shifted point also becomes more uniform w.r.t their surrounding.


\vspace{3mm}
In the following experiments, we employ the same density estimators used in DBSCAN and DP (in clustering) and $k$NN anomaly detector in CDF-TS to transform $D$ to $D'$, i.e., $\epsilon$-neighbourhood density estimator is used in CDF-TS for clustering; and $k$-th nearest neighbour density estimator is used in CDF-TS for anomaly detection.

\section{Empirical Evaluation}
\label{sec_result}

This section presents experiments designed to evaluate the effectiveness of CDF-TS.  

All algorithms used in our experiments were implemented in Matlab R2020a (the CDF-TS source code is available at \url{https://sourceforge.net/p/cdf-ts/}). The experiments were run on a machine with eight cores CPU (Intel Core i7-7820X @ 3.60GHz), 32GB memory and a 2560 CUDA cores GPU (GeForce GTX 1080). All datasets were normalised using the $min$-$max$ normalisation to yield each attribute to be in [0,1] before the experiments began.\footnote{Normalisation is a standard pre-processing step, necessary to prevent variables with large range dominating the distance calculation. Better obviously, would be to have users with domain knowledge set the scale of each variable such that the relative importance of each feature in the distance calculation is appropriate. However, in the absence of such domain knowledge, we default to standard preprocessing techniques.}

For clustering, we used two artificial datasets and 13 real-world datasets with different data sizes and dimensions from the UCI Machine Learning Repository \cite{Dua:2017}.\footnote{Many UCI datasets are originally used for ``classification'' tasks since they have the true class labels. For clustering, there is no unbiased method to manually relabel these real-world dataset without the ground truth. This needs to be handled by community as a whole. To follow previous researches in the clustering area \cite{HDBSCAN2013,lotfi2020density,hou2020density}, we assume that the clusters in these datasets are consistent with their class labels.}  
Table~\ref{dataset} presents the data properties of the datasets. 3L is a 2-dimensional data containing three elongated clusters with different densities, as shown in Figure \ref{data:a}. 4C dataset is a 2-dimensional dataset containing four clusters with different densities (three Gaussian clusters and one elongated cluster), as shown in Figure \ref{data:b}. Note that DBSCAN is unable to correctly identify all clusters in both of these datasets because they do not satisfy the condition specified in Equation \ref{eqn_condition}. Furthermore, clusters in 3L dataset significantly overlap on individual attribute projections, which violates the requirement of ReScale that  one-dimensional projections allow for the identification of the density peaks of each cluster.

\begin{table}[!tb]
 \centering
  \caption{Properties of datasets used for clustering}
    \begin{tabular}{|c|ccc|}
    \hline
    Dataset & Data Size & \#Dimensions & \#Clusters \\
    \hline
     Pendig & 10992  & 16  &  10  \\
      Segment & 2310  & 19    & 7  \\
      Mice  & 1080  & 83    & 8 \\
      Biodeg & 1055  & 41    & 2 \\
      ILPD  & 579   & 9     & 2 \\
      ForestType & 523   & 27    & 4 \\
      Wilt  & 500   & 5     & 2 \\
      Musk  & 476   & 166   & 2 \\
      Libras & 360   & 90    & 15 \\
      Dermatology & 358   & 34    & 6 \\
      Haberman & 306   & 3     & 2 \\
      Seeds & 210   & 7     & 3 \\
      Wine  & 178   & 13    & 3 \\
    \hdashline
    3L    & 560   & 2     & 3 \\
    4C    & 1250  & 2     & 4 \\
    \hline
    \end{tabular}%
  \label{dataset}%
\end{table}%

\begin{figure}[!tb]
\centering
  \begin{subfigure}[b]{0.34\textwidth}
  \centering
    \includegraphics[width=1.7in]{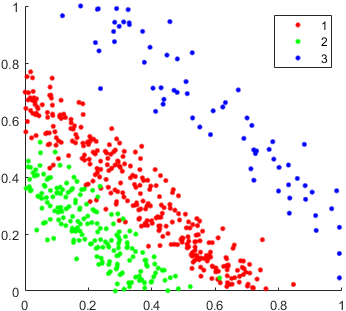}
    \caption{3L data distribution}
    \label{data:a}
  \end{subfigure}  %
  \begin{subfigure}[b]{0.324\textwidth}
  \centering
    \includegraphics[width=1.7in]{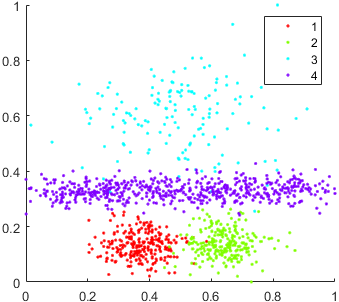}
    \caption{4C data distribution}
        \label{data:b}
  \end{subfigure}
\caption{(a) The scatter plot of a two-dimensional data containing three elongated clusters. (b) The scatter plot of a two-dimensional data containing four clusters.} 
    \label{data} 
\end{figure}

For anomaly detection, we compared the anomaly detection performance on two synthetic datasets with anomalous clusters and 10 real-world benchmark datasets\footnote{ Velocity, Ant and Tomcat are from \url{http://openscience.us/repo/defect/ck/} and others are from UCI Machine Learning Repository \cite{Dua:2017}. For Mfeat and Vowel, digit 0 and label 1 are assigned as the anomaly class, respectively. For Dermatology, class 6 is assigned to anomaly. For all other datasets, we assign the most-represented class as “normal” and all other classes as “anomaly". }. The data size, dimensions and percentage of anomalies are shown in Table~\ref{anoData}. Both the Syn 1 and Syn 2 datasets contain clusters of anomalies. Their data distributions are shown in Figure~\ref{syn}.

\begin{table}[!htb]
  \centering
  \caption{Properties of datasets used for anomaly detection}
    \begin{tabular}{|c|ccc|}
    \hline
    Dataset  & Data Size & \#Dimensions &  \% Anomaly \\    \hline
    AnnThyroid & 7200  & 6     & 7.4\%   \\
    Pageblocks & 5473  & 10    & 10.2\%  \\
    Tomcat & 858   & 20    & 9.0\%   \\
    Ant   & 745   & 20    & 22.3\% \\
    BloodDonation & 604   & 4     & 5.6\% \\
    Vowel & 528   & 10    & 9.1\% \\
    Mfeat & 410   & 649   & 2.4\% \\
    Dermatology & 358   & 34    & 5.6\% \\
    Balance & 302   & 4     & 4.6\% \\
    Velocity & 229   & 20    & 34.1\%  \\
      \hdashline
  Syn 1 & 520   & 2     & 3.9\% \\
  Syn 2 & 860   & 1     & 7.0\% \\
    \hline
    \end{tabular}%
  \label{anoData}%
\end{table}%

\begin{figure}[!htb]
\centering
  \begin{subfigure}[b]{0.4\textwidth}
  \centering
    \includegraphics[width=0.8\textwidth]{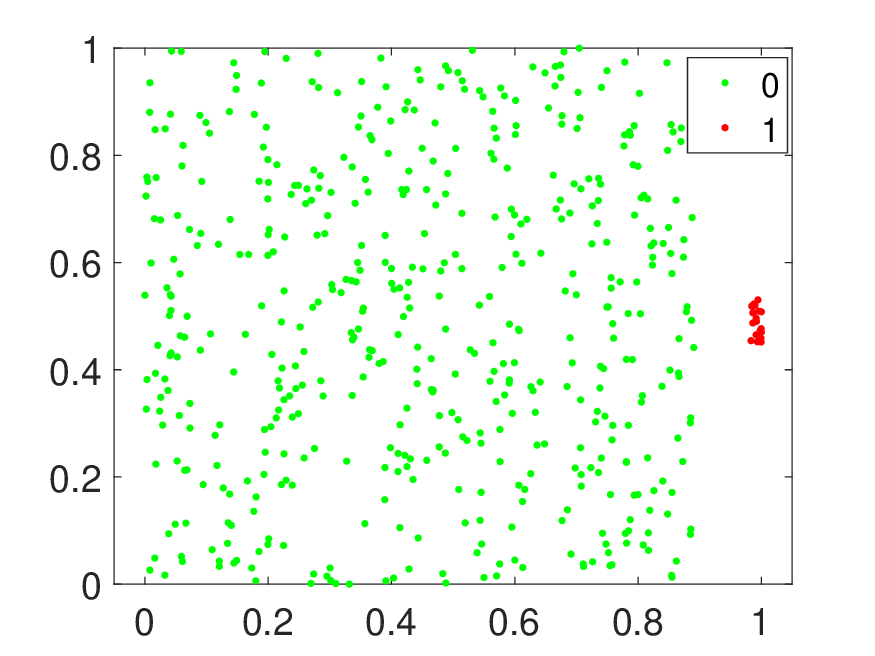}
    \caption{The scatter plot of Syn 1 dataset}
    \label{syn:a}
  \end{subfigure}  \hspace{5mm}%
  \begin{subfigure}[b]{0.4\textwidth}
  \centering
    \includegraphics[width=0.8\textwidth]{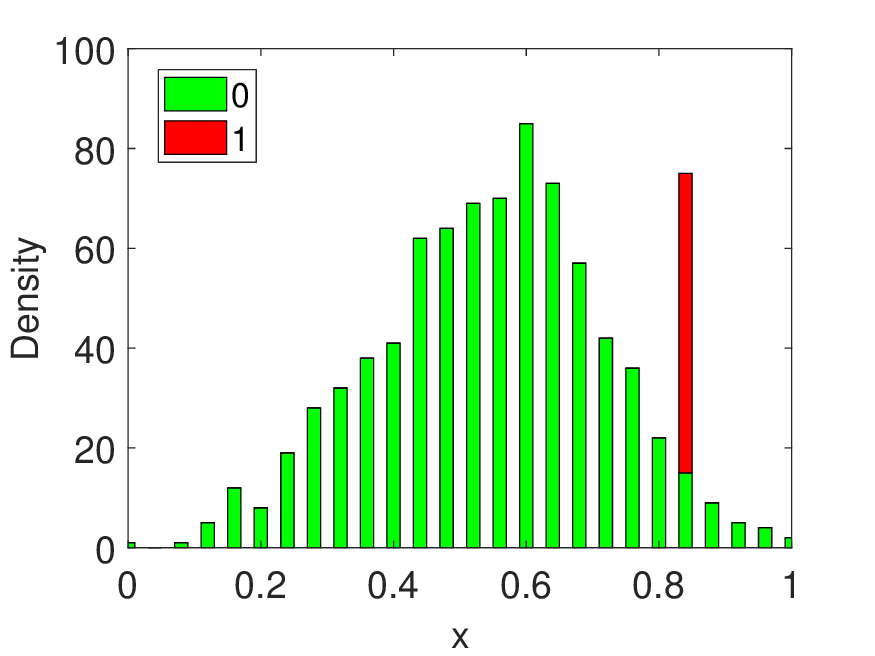}
    \caption{The stacked histogram on Syn 2 dataset}
        \label{syn:b}
  \end{subfigure}
\caption{Data distributions of the Syn 1 and Syn 2 datasets, where the red colour indicates the anomaly.} 
    \label{syn} 
\end{figure}


\subsection{Clustering}
\label{bestClu}
 
In this section, we compare CDF-TS with ReScale and DScale using four existing density-based clustering algorithms, i.e., DBSCAN \cite{ester1996density}, DP \cite{rodriguez2014clustering} , LGD \cite{li2019local} and EC \cite{wang2020extreme}, in terms of F-measure \cite{Fmeasure}: given a clustering result, we calculate the precision score $P_{i}$ and the recall score $R_{i}$ for each cluster $C_{i}$ based on the confusion matrix, and then the F-measure score of $C_{i}$ is the harmonic mean of $P_{i}$ and $R_{i}$. The overall F-measure score is the unweighted (macro) average over all clusters: F-measure$=\frac{1}{\varsigma}\sum_{i=1}^{\varsigma}\frac{2P_{i}R_{i}}{P_{i}+R_{i}}$.\footnote{We use the Hungarian algorithm \cite{kuhn1955hungarian} to search the optimal match between the clustering results and true clusters. The unmatched clusters/instances will be treated as noise which affects the recall score in the F-measure calculation. It is worth noting that other evaluation measures such as purity, Normalised Mutual Information (NMI) \cite{strehl2002cluster} and Adjusted Mutual Information (AMI) \cite{vinh2010information} only take into account the points assigned to clusters and do not consider noise. A clustering algorithm which assigns the majority of the points to noise could potentially result in a misleadingly high clustering performance. Thus, the F-measure is more suitable than purity or AMI in assessing the clustering performance of density-based clustering when there are points assigned to noise.   We have provided the AMI results of different versions of DP and LGD clustering in  \ref{appC}, because both algorithms can generate cluster labels for all points. The AMI results also show similar trends to Table \ref{best}, i.e., CDF-TS version is significantly better than other versions}

We report the best clustering performance after a parameter search for each algorithm. Table \ref{para} lists the parameters and their search ranges for each algorithm. $\psi$ in ReScale controls the precision of $\widehat{cdf}_{\lambda}(x)$, i.e., the number of intervals used for estimating $\widehat{cdf}_{\lambda}(x)$. We set $\delta=0.015$ as the default value for CDF-TS.

\begin{table}[!htb]
 \renewcommand{\arraystretch}{0.9}
 \setlength{\tabcolsep}{4.pt}
  \centering
  \caption{Parameters and their search ranges. 
  The search ranges of $\psi$ and $\lambda$ are as used by Zhu et. al. \cite{ZHU2016983}.  The parameter search ranges for LGD follow the original paper \cite{li2019local}}. 
    \begin{tabular}{|c|c|}
    \hline
    Algorithm & Parameters with search ranges \\
    \hline
    DBSCAN & $Minpts \in \{ 2,3,...,10 \}$; $\epsilon \in [0,1]$ \\
    DP & $k \in \{ 2,3,...,20 \}$; $\epsilon \in [0,1]$ \\
    EC &  $\epsilon \in [0,1]$ \\
    LGD  & $k \in \{ 2,4, 6, 10, 15, 20 \}$; $\tau \in [0.5, 0.52, 0.56, 0.62]$; $c=\#clusters$ \\
\hdashline 
ReScale & $\psi=100  $; $\lambda\in \lbrace0.1,0.2,..., 0.5\rbrace$\\
DScale &   $\lambda \in \lbrace0.1,0.2,..., 0.5\rbrace$\\
CDF-TS &   $\lambda\in \lbrace0.1,0.2,..., 0.5\rbrace$; $\delta =0.015$ \\
    \hline
    \end{tabular}%
  \label{para}%
\end{table}%

Table \ref{best} shows the best F-measures for DBSCAN, DP, EC, LGD and their  ReScale, DScale and CDF-TS versions. The average F-measures, showed in the second last row, reveal that CDF-TS improves the clustering performance of all clustering algorithms with a larger performance gap than both ReScale and DScale. In addition, CDF-TS is the best performer on many more datasets than other contenders (shown in the last row of Table \ref{best}.)

\begin{landscape}
  \begin{table}[!p]
  \renewcommand{\arraystretch}{1.2}
 \setlength{\tabcolsep}{2.6pt}
    \centering
    \caption{ The best F-measure for DBSCAN, DP, EC, LGD and their ReScale, DScale and CDF-TS versions. For each clustering algorithm, the best performer in each dataset is boldfaced. Ori, ReS, and DS represent the Original algorithm, ReScale, and DScale respectively.} 
      \begin{tabular}{|c|cccc|cccc|cccc|cccc|}
      \hline
      \multirow{2}[3]{*}{Data} & \multicolumn{4}{c|}{DBSCAN}   & \multicolumn{4}{c|}{DP}       & \multicolumn{4}{c|}{EC}       & \multicolumn{4}{c|}{LGD} \\
\cline{2-17}            & Orig  & ReS   & DS    & CDF-TS  & Orig  & ReS   & DS    & CDF-TS  & Orig  & ReS   & DS    & CDF-TS  & Orig  & ReS   & DS    & CDF-TS  \\
   \hline 
      Pendig & 0.70  & 0.78  & 0.74  & \textbf{0.80} & 0.79  & 0.82  & 0.82  & \textbf{0.84} & 0.79  & 0.84  & 0.84  & \textbf{0.844} & 0.83  & 0.878 & 0.87  & \textbf{0.88} \\
      Segment & 0.59  & 0.62  & 0.61  & \textbf{0.67} & 0.78  & 0.77  & 0.80  & \textbf{0.84} & 0.57  & 0.66  & \textbf{0.69} & 0.66  & 0.76  & 0.73  & 0.76  & \textbf{0.80} \\
      Mice  & 0.99  & 0.99  & 0.99  & \textbf{0.993} & \textbf{1.00} & \textbf{1.00} & \textbf{1.00} & \textbf{1.00} & 0.98  & 0.98  & \textbf{0.99} & \textbf{0.99} & \textbf{1.00} & \textbf{1.00} & \textbf{1.00} & \textbf{1.00} \\
      Biodeg & 0.45  & 0.44  & 0.47  & \textbf{0.52} & 0.72  & 0.74  & 0.73  & \textbf{0.76} & 0.45  & 0.44  & 0.61  & \textbf{0.70} & 0.68  & 0.70  & 0.71  & \textbf{0.73} \\
      ILPD  & 0.41  & 0.42  & \textbf{0.56} & 0.52  & 0.60  & 0.63  & 0.62  & \textbf{0.64} & 0.49  & 0.49  & \textbf{0.58} & 0.57  & 0.57  & 0.62  & 0.57  & \textbf{0.600} \\
      ForestType & 0.27  & 0.51  & 0.48  & \textbf{0.65} & 0.69  & 0.83  & 0.70  & \textbf{0.85} & 0.28  & 0.66  & 0.55  & \textbf{0.84} & 0.63  & 0.82  & 0.80  & \textbf{0.84} \\
      Wilt  & 0.38  & 0.39  & 0.54  & \textbf{0.55} & 0.54  & 0.68  & 0.54  & \textbf{0.74} & 0.44  & 0.53  & 0.49  & \textbf{0.67} & 0.53  & 0.56  & 0.54  & \textbf{0.71} \\
      Musk  & 0.51  & 0.52  & 0.51  & \textbf{0.53} & 0.55  & 0.58  & 0.61  & \textbf{0.62} & 0.51  & 0.53  & 0.54  & \textbf{0.53} & 0.60  & 0.61  & 0.63  & \textbf{0.586} \\
      Libras & 0.41  & 0.44  & 0.46  & \textbf{0.49} & 0.52  & 0.52  & 0.52  & \textbf{0.53} & 0.46  & 0.46  & 0.49  & \textbf{0.49} & 0.48  & 0.51  & 0.49  & \textbf{0.54} \\
      Dermatology & 0.52  & 0.73  & 0.74  & \textbf{0.83} & 0.91  & \textbf{0.97} & 0.91  & 0.96  & 0.73  & 0.76  & 0.75  & \textbf{0.87} & 0.92  & \textbf{0.97} & 0.94  & 0.96 \\
      Haberman & 0.47  & 0.64  & 0.59  & \textbf{0.66} & 0.56  & 0.63  & 0.58  & \textbf{0.67} & 0.49  & 0.58  & 0.54  & \textbf{0.66} & 0.53  & \textbf{0.67} & 0.56  & 0.664 \\
      Seeds & 0.75  & \textbf{0.88} & 0.85  & 0.83  & 0.91  & 0.92  & 0.92  & \textbf{0.94} & 0.90  & 0.90  & 0.91  & \textbf{0.88} & 0.91  & 0.91  & 0.91  & \textbf{0.91} \\
      Wine  & 0.64  & 0.86  & 0.80  & \textbf{0.90} & 0.93  & 0.95  & 0.96  & \textbf{0.962} & 0.84  & \textbf{0.92} & 0.90  & 0.88  & 0.92  & 0.94  & 0.92  & \textbf{0.96} \\ \hdashline
      3L    & 0.59  & 0.63  & \textbf{0.90} & 0.88  & 0.82  & 0.81  & 0.86  & \textbf{0.89} & 0.76  & \textbf{0.79} & 0.70  & 0.73  & 0.71  & 0.71  & 0.71  & \textbf{0.70} \\
      4C    & 0.71  & 0.90  & 0.92  & \textbf{0.95} & 0.87  & 0.92  & \textbf{0.95} & 0.91  & 0.86  & 0.88  & 0.88  & \textbf{0.883} & 0.95  & \textbf{0.961} & 0.95  & 0.96  \\
      \hline
      Average & 0.56  & 0.65  & 0.68  &  {0.72} & 0.75  & 0.78  & 0.77  &  {0.81} & 0.64  & 0.70  & 0.70  &  {0.75} & 0.73  & 0.77  & 0.76  &  {0.79}  \\
      \hline
      \#Top 1 & 0     & 1     & 2     & 12    & 1     & 2     & 2     & 13    & 0     & 2     & 3     & 11    & 1     & 4     & 1     & 12  \\
      \hline 
      \end{tabular}%
       \label{best}%
  \end{table}%
\end{landscape}

It is interesting to see that CDF-TS performs notably better than both DBSCAN and DP on many datasets, such as ForestType, Wilt, Dermatology and Wine. On most datasets, CDF-TS also outperforms its competitors ReScale and/or DScale by a wide gap.

Note that the performance gap is smaller for DP than it is for DBSCAN because DP is a more powerful algorithm which does not rely on a single density threshold to identify clusters.

To evaluate whether the performance difference among the three scaling algorithms is significant, we conduct the Friedman test with the post-hoc Nemenyi test \cite{demvsar2006statistical}.  Figure~\ref{sig} shows the results of the significance test for different clustering algorithms. 	The results show that the CDF-TS versions are significantly better than at least one of DScale and ReScale versions for each clustering algorithm. 
 
 \begin{figure}[!htb]
  	\centering
  \begin{subfigure}[b]{0.45\textwidth}
  \centering
    \includegraphics[width=2.6in]{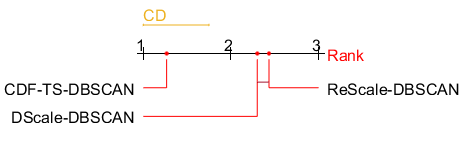}
    \caption{Significance test for DBSCAN}  \vspace{3mm}%
    \label{sig:a}
  \end{subfigure}  
  \begin{subfigure}[b]{0.45\textwidth}
  \centering
    \includegraphics[width=2.3in]{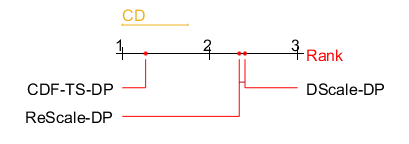}
    \caption{Significance test for DP}
        \label{sig:b}
  \end{subfigure}
  \begin{subfigure}[b]{0.45\textwidth}
  \centering
    \includegraphics[width=2.3in]{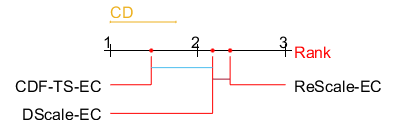}
    \caption{Significance test for EC}
        \label{sig:c}
  \end{subfigure}
  \begin{subfigure}[b]{0.45\textwidth}
  \centering
    \includegraphics[width=2.3in]{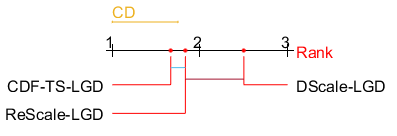}
    \caption{Significance test for LGD}
        \label{sig:d}
  \end{subfigure}
  	\caption{Critical difference (CD) diagram of the post-hoc Nemenyi test ($\alpha=0.10$). Two algorithms are significantly different if the gap between their ranks is larger than the CD. Otherwise, there is a line linking them.}
  	\label{sig}
  \end{figure}

 Here we compare the different effects on the transformed datasets due to ReScale and CDF-TS. The effects on the two synthetic datasets are shown in Figure~\ref{DESSyth} and Figure \ref{RScaleSyth}. They show that both the rescaled datasets are more axis-parallel distributed using ReScale than those using CDF-TS. For the 3L dataset, the blue cluster is still very sparse after running ReScale, as shown in Figure \ref{DESSyth:a}. In contrast, it becomes denser using CDF-TS, as shown in Figure \ref{RScaleSyth:a}. As a result, DBSCAN has much better performance with CDF-TS on the 3L dataset. Figure \ref{wilt} shows the MDS visualisation\footnote{Multidimensional scaling (MDS) \cite{borg2012applied} is used for visualising a high-dimensional dataset in a 2-dimensional space through a projection method which preserves as well as possible the original pairwise dissimilarities between instances.} on the Wilt dataset  where CDT-TS outperforms both DBSCAN and DP the most. The figure shows that the each cluster is more homogeneously distributed in the new space, which makes the boundary between clusters easier to be identified than that in the original space. In contrast, based on distance, most points of the red cluster are surrounded by points of the blue cluster, as shown in Figure \ref{wilt:a}. 

\begin{figure}[!htb]
\centering
  \begin{subfigure}[b]{0.43\textwidth}
  \centering
    \includegraphics[width=0.8\textwidth]{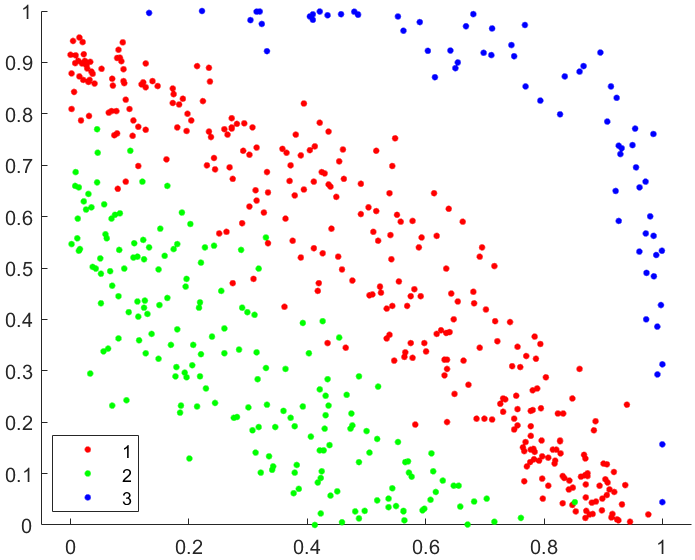}
    \caption{3L data distribution}
    \label{DESSyth:a}
  \end{subfigure}  %
  \begin{subfigure}[b]{0.43\textwidth}
  \centering
    \includegraphics[width=0.8\textwidth]{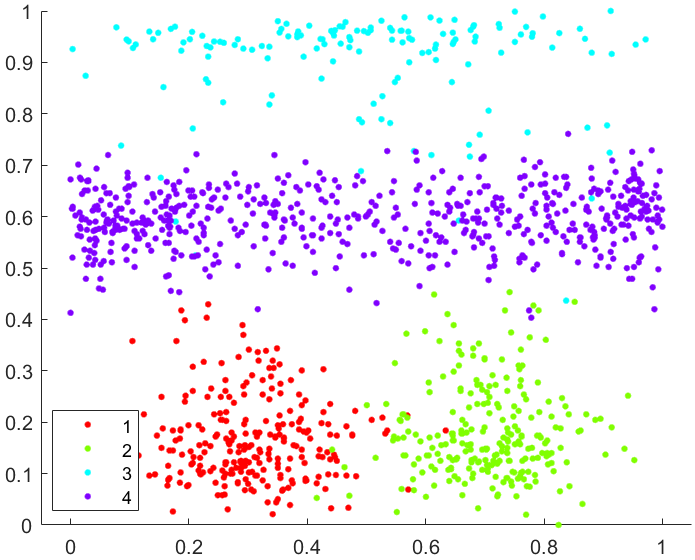}
    \caption{4C data distribution}
        \label{DESSyth:b}
  \end{subfigure}
\caption{Scatter plots after running ReScale when achieving the best DBSCAN clustering performance} 
    \label{DESSyth} 
\end{figure}

\begin{figure}[!htb]
\centering
  \begin{subfigure}[b]{0.43\textwidth}
  \centering
    \includegraphics[width=0.8\textwidth]{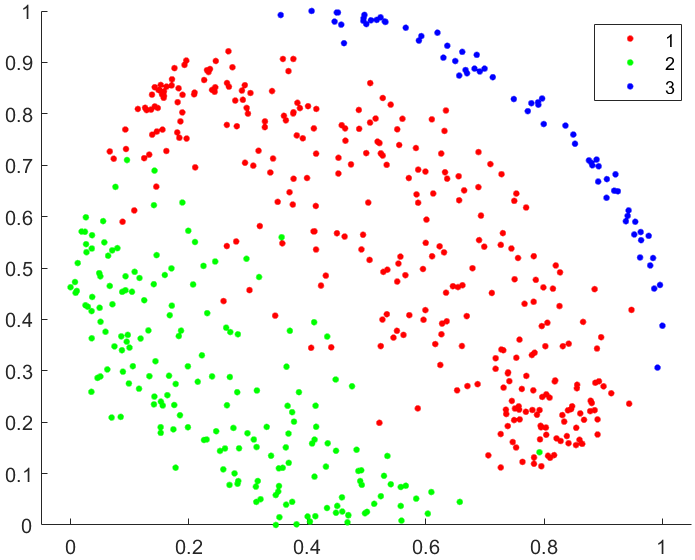}
    \caption{3L data distribution}
    \label{RScaleSyth:a}
  \end{subfigure}  
  \begin{subfigure}[b]{0.43\textwidth}
  \centering
    \includegraphics[width=0.8\textwidth]{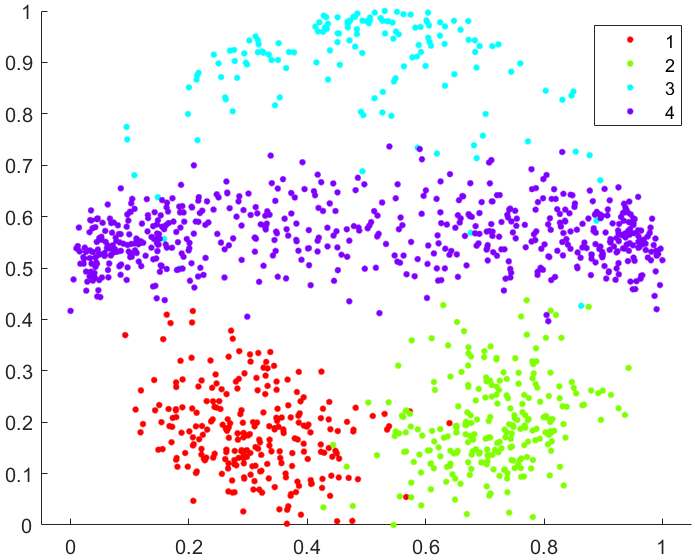}
    \caption{4C data distribution}
        \label{RScaleSyth:b}
  \end{subfigure}
\caption{Scatter plots after running CDF-TS when achieving the best DBSCAN clustering performance} 
    \label{RScaleSyth} 

\centering
  \begin{subfigure}[b]{0.42\textwidth}
  \centering
    \includegraphics[width=0.86\textwidth]{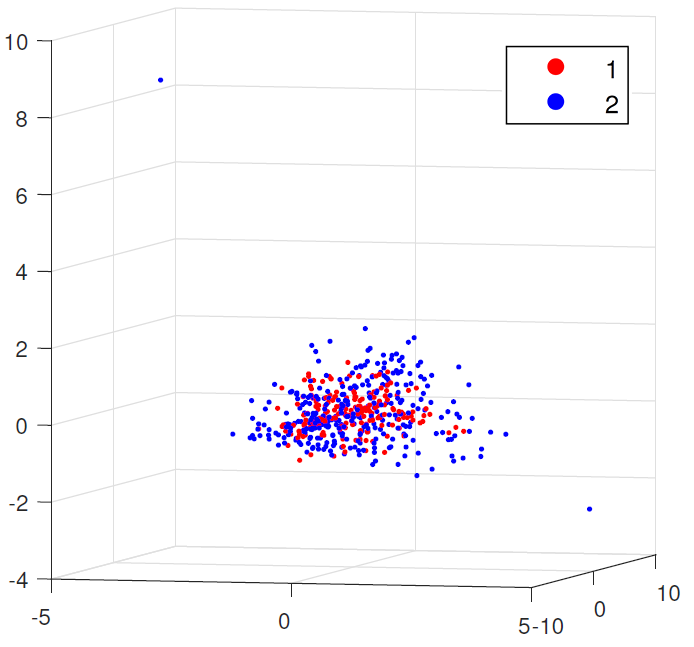}
    \caption{Based on distance}
    \label{wilt:a}
  \end{subfigure} 
  \begin{subfigure}[b]{0.42\textwidth}
  \centering
    \includegraphics[width=0.95\textwidth]{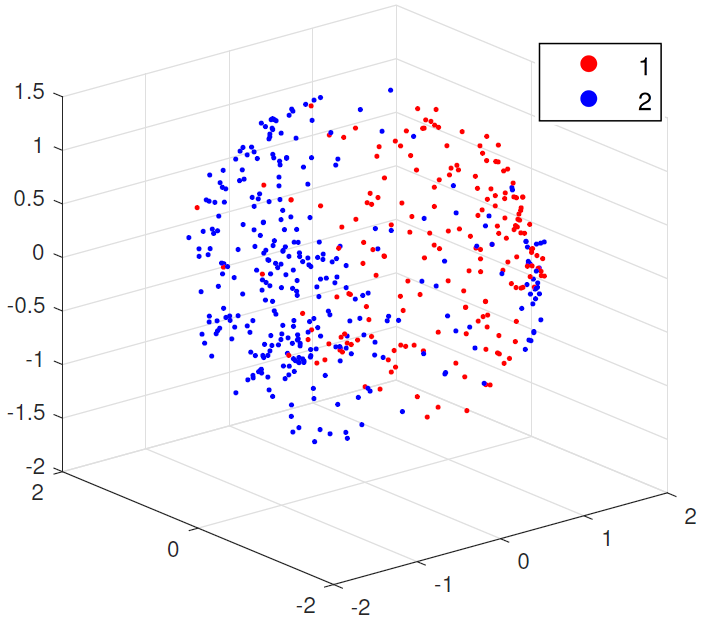}
    \caption{Based on CDF-TS}
        \label{wilt:b}
  \end{subfigure}
\caption{MDS plots on the Wilt dataset} 
    \label{wilt} 
\end{figure}

\subsection{Anomaly detection}

In this section, we evaluate the ability of CDF-TS to detect local anomalies based on $k$NN anomaly detection. 

Three state-of-the-art anomaly detectors, Local Outlier Factor (LOF) \cite{LOF2000}, iForest \cite{liu2008isolation} and iNNE \cite{bandaragoda2018isolation}, are also used in the comparison. Table \ref{para2} lists the parameters and their search ranges for each algorithm. Parameters $\psi$ and $t$ in iForest control the sub-sample size and number of iTrees, respectively.  We report the best performance of each algorithm on each dataset in terms of best AUC (Area under the Curve of ROC) \cite{fawcett2006introduction}. 

\begin{table}[!htb]
  \renewcommand{\arraystretch}{1.2}
 \setlength{\tabcolsep}{1.1pt}
  \centering
  \caption{Parameters and their search ranges.}
    \begin{tabular}{|c|c|}
    \hline
    Algorithm & Parameters and their search ranges \\
    \hline
    $k$NN/LOF & $k \in \{5\%n, 10\%n,...,50\%n \}$\\
    iForest/iNNE & $t=100$; $\psi \in \{2^1, 2^2,..., 2^{10}\}$ \\
\hdashline 
ReScale & $\psi=100  $; $\lambda\in \lbrace0.1,0.2,..., 0.5\rbrace$\\
DScale &   $\lambda\in \lbrace0.1,0.2,..., 0.5\rbrace$\\
CDF-TS &   $\lambda\in \lbrace0.1,0.2,..., 0.5\rbrace$; $\delta =0.015$ \\
    \hline
    \end{tabular}%
  \label{para2}%
\end{table}%

Table \ref{best2} compares the best AUC score of each algorithm. CDF-TS-$k$NN achieves the highest average AUC of 0.90 and performs the best on 5 out of 12 datasets. For the Syn 1 dataset which has significant overlapping on the individual attribute projection, iForest and ReScale perform the worst because they are based on individual attribute splitting and scaling, respectively. All of ReScale, DScale and CDF-TS improve the performance of $k$NN on Syn 1 dataset which has cluster anomalies, because these anomalies become farther from the normal cluster centre. It is interesting to mention that CDF-TS-$k$NN has the largest AUC increase in $k$NN on Dermatology and AnnThyroid.

  \begin{table}[!tb]
  \renewcommand{\arraystretch}{1.2}
 \setlength{\tabcolsep}{3.pt}
    \centering
  \caption{The best AUC on 12 datasets. The best performer on each dataset is boldfaced. Original, ReScale, DScale and CDF-TS represent $k$NN, ReScale-$k$NN, DeScale-$k$NN and CDF-TS-$k$NN, respectively.}
      \begin{tabular}{|c|ccccccc|}
      \hline
      {\multirow{2}[3]{*}{Data}}     & \multirow{2}[3]{*}{LOF}   & \multirow{2}[3]{*}{iForest} & \multirow{2}[3]{*}{iNNE}  & \multicolumn{4}{c|}{$k$NN}  \\\cline{5-8}   &   &   &  & Original & ReScale   & DScale    & CDF-TS  \\
      \hline
      AnnThyroid & 0.68  & 0.88  & 0.74  & 0.65  & 0.76  & 0.71  & \textbf{0.94} \\
      Pageblocks & 0.93  & 0.90  & 0.92  & 0.89  & 0.88  & 0.86  & \textbf{0.94} \\
      Tomcat & 0.67  & \textbf{0.81} & 0.69  & 0.63  & 0.77  & 0.69  & 0.78 \\
      Ant   & 0.68  & \textbf{0.77} & 0.70  & 0.67  & 0.76  & 0.71  & 0.76 \\
      BloodDonation & 0.79  & 0.82  & 0.74  & 0.69  & 0.84  & 0.80  & \textbf{0.86} \\
      Vowel & 0.93  & 0.92  & \textbf{0.95} & 0.93  & 0.930 & 0.90  & 0.94 \\
      Mfeat & 0.98  & 0.98  & 0.99  & 0.99  & 0.99  & \textbf{1.00} & 0.99 \\
      Dermatology & 0.99  & 0.86  & 0.98  & 0.91  & 0.99  & 0.97  & \textbf{1.00} \\
      Balance & 0.95  & 0.91  & \textbf{0.97} & 0.94  & 0.91  & 0.83  & 0.92 \\
      Velocity & 0.62  & 0.65  & 0.61  & 0.66  & 0.67  & \textbf{0.694} & 0.69 \\ \hdashline
      Syn 1 & 0.93  & 0.90  & 0.94  & 0.92  & 0.89  & 0.95  & \textbf{0.98} \\
      Syn 2 & 0.88  & 0.86  & 0.91  & 0.88  & 0.92  & \textbf{0.95} & 0.94 \\       \hline
      \textit{Average} & 0.83  & 0.86  & 0.85  & 0.81  & 0.86  & 0.84  & 0.90 \\   \hline
      \#Top 1 & 0     & 2     & 2     & 0     & 0     & 3     & 5 \\ 
      \hline
\end{tabular}%
  \label{best2}%
  \end{table}%

Figure \ref{sig2} shows the significance test on the three algorithms applied to $k$NN anomaly detector. It can be seen from the results that CDF-TS-$k$NN is significantly better than DScale-$k$NN and ReScale-$k$NN.

 \begin{figure}[!tb]
  	\centering
  \begin{subfigure}[b]{0.45\textwidth}
  \centering
    \includegraphics[width=2.9in]{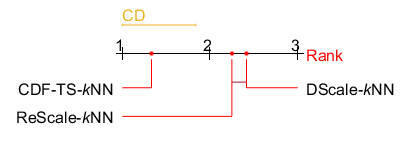}
  \end{subfigure}
  	\caption{Critical difference (CD) diagram of the post-hoc Nemenyi test ($\alpha=0.10$) for anomaly detection algorithms.}
  	\label{sig2}
  \end{figure}



\subsection{Compared with metric learning algorithms}

We have examined the best performance of PCA \cite{jolliffe2011principal} and t-SNE \cite{maaten2008visualizing} in DP clustering and $k$NN anomaly detection, as shown in Table \ref{tab:DP} and Table \ref{tab:knn}, respectively. We set the dimensionality of PCA and t-SNE output space to 2 or $d$. The perplexity in t-SNE are searched in $[10, 20, 30, 40 ,50]$. We find that PCA could not improve the performance on most of these tasks. t-SNE-DP has comparable clustering performance to CDF-TS-DP. However, t-SNE could not improve the performance on $k$NN anomaly detector on some datasets, such as the AnnThyroid,  Pageblocks and Tomcat datasets.  

  \begin{table}[!tb]
    \centering
    \caption{The best F-measure of DP and its PCA, t-SNE and CDF-TS versions. For each version, the best performer in each dataset is boldfaced.}
      \begin{tabular}{|c|cccc|}
      \hline
      Dataset  & \multicolumn{4}{c|}{DP}\\\cline{2-5}
      & Original & PCA   & t-SNE & CDF-TS   \\  \hline
      Segment & 0.78  & 0.78  & 0.81  & \textbf{0.84} \\
      Mice  & \textbf{1.00} & \textbf{1.00} & \textbf{1.00} & \textbf{1.00} \\
      Biodeg & 0.72  & 0.72  & 0.73  & \textbf{0.76} \\
      ILPD  & 0.60  & 0.60  & 0.61  & \textbf{0.64} \\
      ForestType & 0.69  & 0.75  & \textbf{0.87} & 0.85  \\
      Wilt  & 0.54  & 0.55  & 0.60  & \textbf{0.74} \\
      Musk  & 0.55  & 0.55  & 0.58  & \textbf{0.62} \\
      Libras & 0.52  & 0.52  & \textbf{0.56} & 0.53  \\
      Dermatology & 0.91  & 0.91  & \textbf{0.97} & 0.96  \\
      Haberman & 0.56  & 0.56  & 0.57  & \textbf{0.67} \\
      Seeds & 0.91  & 0.93  & \textbf{0.94} & \textbf{0.94} \\
      Wine  & 0.93  & \textbf{0.98} & 0.97  & 0.96  \\       \hdashline
      3L    & 0.82  & 0.70  & \textbf{0.91} & 0.89  \\
      4C    & 0.87  & 0.87  & \textbf{0.96} & 0.91   \\
      \hline
      \textit{Average} & 0.74  & 0.74  & 0.79  & 0.81  \\
      \hline
      {\#Top 1} & 1     & 2     & 7     & 8 \\
      \hline
      \end{tabular}%
    \label{tab:DP}%
  \end{table}%

  \begin{table}[!tb]
    \centering
    \caption{The best AUC on 12 datasets. The best performer on each dataset is boldfaced. PCA, t-SNE and CDF-TS represent PCA-$k$NN, t-SNE-$k$NN and CDF-TS-$k$NN, respectively.}
      \begin{tabular}{|c|cccc|}
      \hline
      {\multirow{2}[3]{*}{Data}} & \multicolumn{4}{c|}{$k$NN}\\ \cline{2-5}
      &  Original  & PCA   & t-SNE & CDF-TS   \\  \hline
      AnnThyroid & 0.65  & 0.65  & 0.67  & \textbf{0.94} \\
      Pageblocks & 0.89  & 0.89  & 0.88  & \textbf{0.94} \\
      Tomcat & 0.63  & 0.63  & 0.60  & \textbf{0.78} \\
      Ant   & 0.67  & 0.67  & 0.68  & \textbf{0.76} \\
      BloodDonation & 0.69  & 0.75  & 0.75  & \textbf{0.86} \\
      Vowel & 0.93  & 0.93  & \textbf{0.97} & 0.94 \\
      Mfeat & 0.99  & 0.99  & \textbf{1.00} & 0.99 \\
      Dermatology & 0.91  & 0.91  & \textbf{1.00} & \textbf{1.00} \\
      Balance & 0.94  & \textbf{0.96} & 0.94  & 0.92 \\
      Velocity & 0.66  & 0.66  & 0.67  & \textbf{0.69} \\       \hdashline
      Syn 1 & 0.92  & 0.92  & \textbf{1.00} & 0.98 \\
      Syn 2 & 0.88  & 0.88  & \textbf{0.96} & 0.94 \\
      \hline
      \textit{Average} & 0.81  & 0.82  & 0.84  & 0.90  \\
      \hline
      \#Top 1 & 0     & 1     & 5   & 7  \\
      \hline
      \end{tabular}%
    \label{tab:knn}%
  \end{table}%

In order to visually compare the effects of these transform methods, Table \ref{MDSa} shows the visualisation results on two datasets. It shows that most anomalies become anomalous clusters and farther away from normal points on the CDF-TS transformed datasets, and thus easier to be detected by the $k$NN anomaly detector. When using PCA, some normal points are still scattered having a large $k$NN distance from other normal points on the Dermatology dataset, so it has the same AUC score to that obtained on the original dataset. On the AnnThyroid dataset, the clustered anomalies  are still close to normal clusters in the t-SNE transformed space. With CDF-TS, the normal points occur on a sphere and there is a large gap between the surface of the sphere and the anomalies.

\begin{table} [!tb] 
  \renewcommand{\arraystretch}{1}
 \setlength{\tabcolsep}{4pt}
\centering
\caption{Visualisation results on two datasets, where red points indicate anomalies. Distance and CDF-TS results are based on MDS, $\eta=0.1$ and $\delta=0.015$ in CDF-TS; PCA and t-SNE results set output dimensionality to 3, and perplexity in t-SNE is 30.} 
  \begin{tabular}{|c|cccc|}
    \hline
     &  Distance  & PCA & t-SNE &  CDF-TS    \\
      \hline
    \begin{turn}{90}   \ \ \ \ Dermatology \end{turn}  & \includegraphics[width=1.37in]{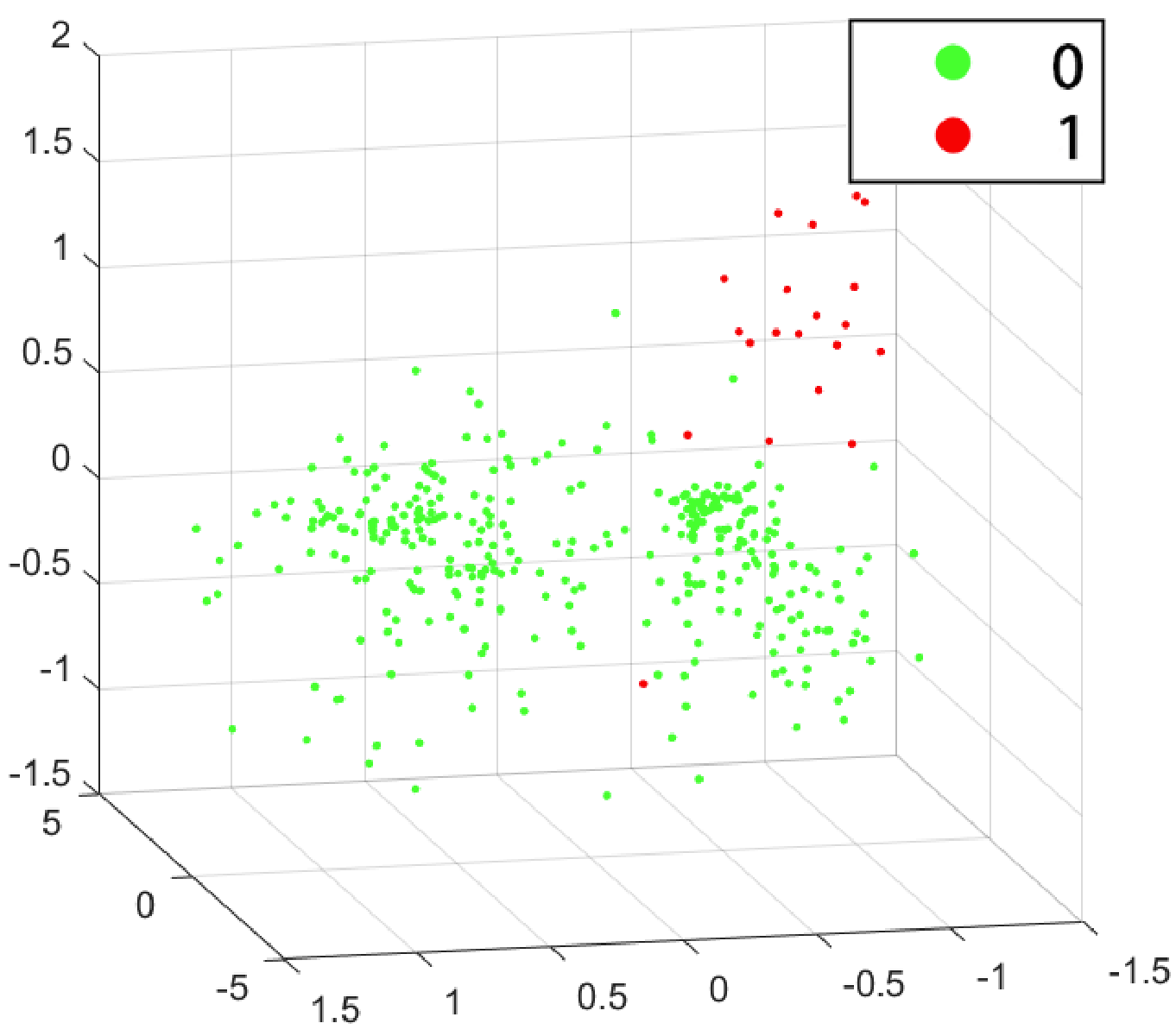} &       \includegraphics[width=1.37in]{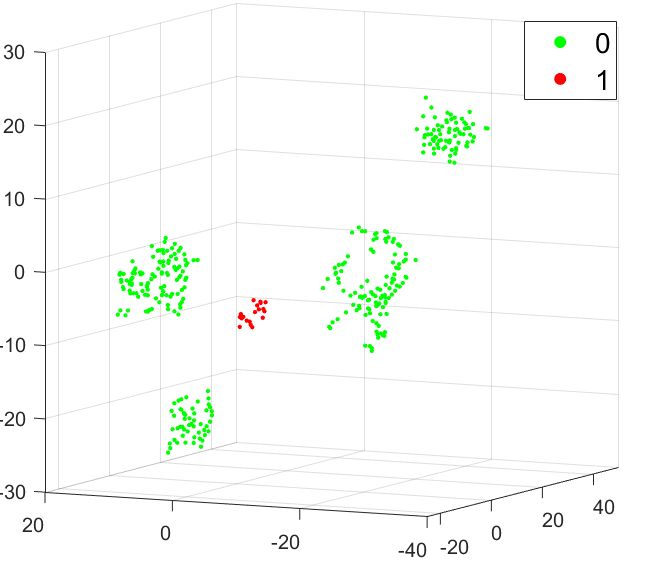}  &       \includegraphics[width=1.37in]{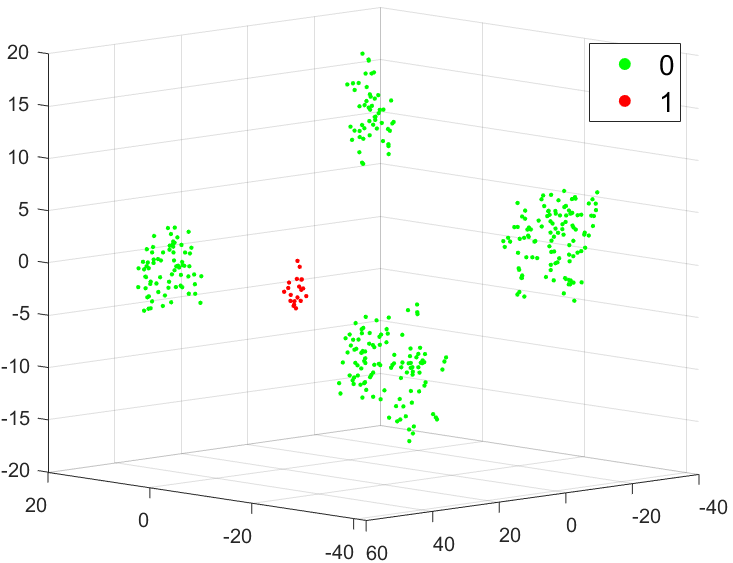}  &
      \includegraphics[width=1.37in]{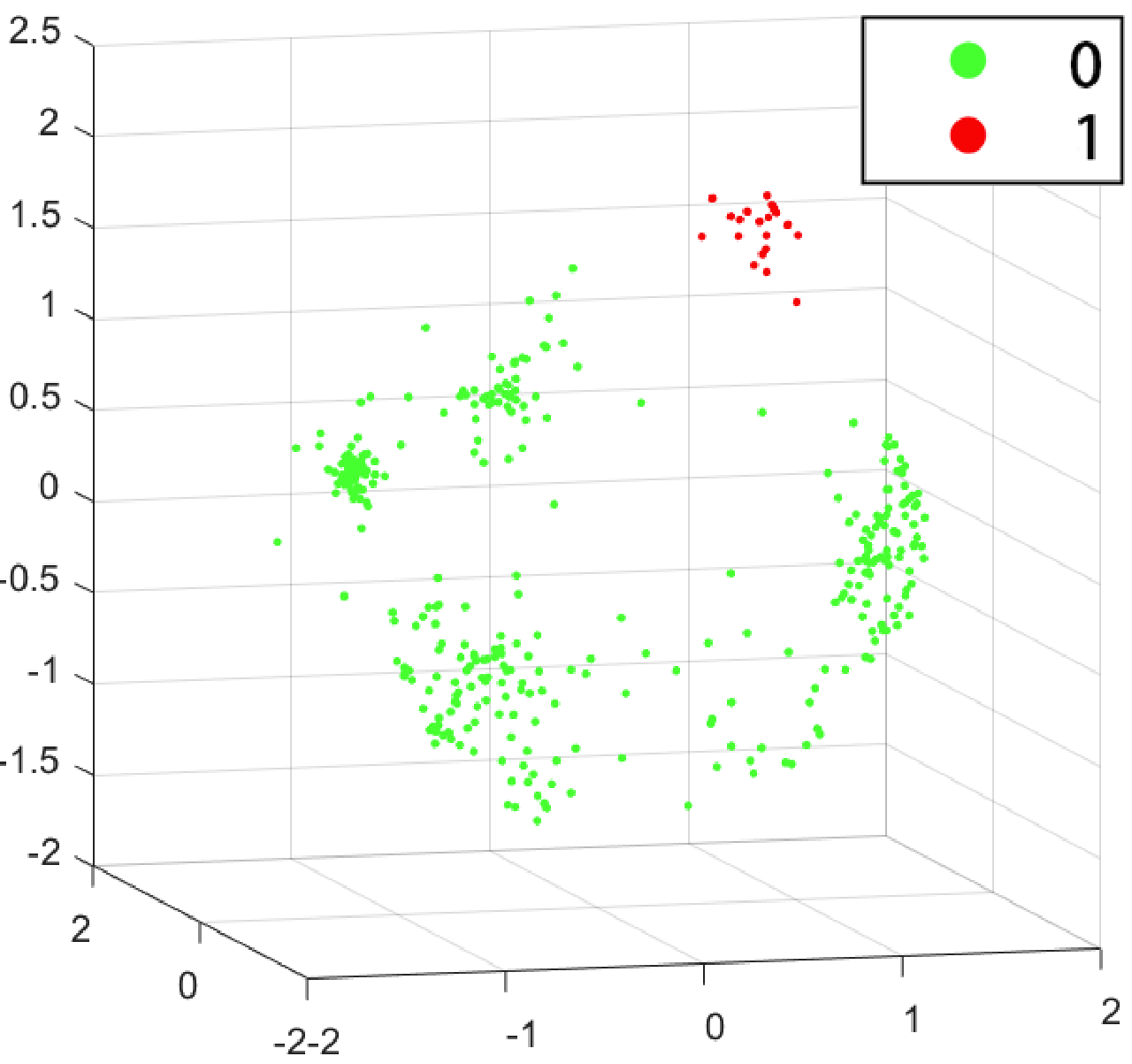}  \\
      \hdashline 
     \begin{turn}{90} \  \ \   \ \ AnnThyroid \end{turn}  & 
      \includegraphics[width=1.37in]{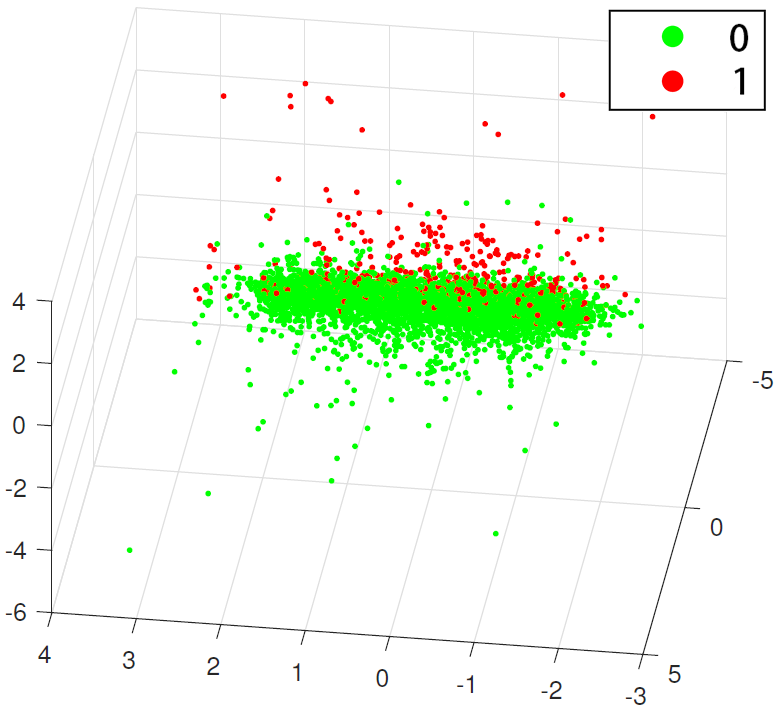} & \includegraphics[width=1.37in]{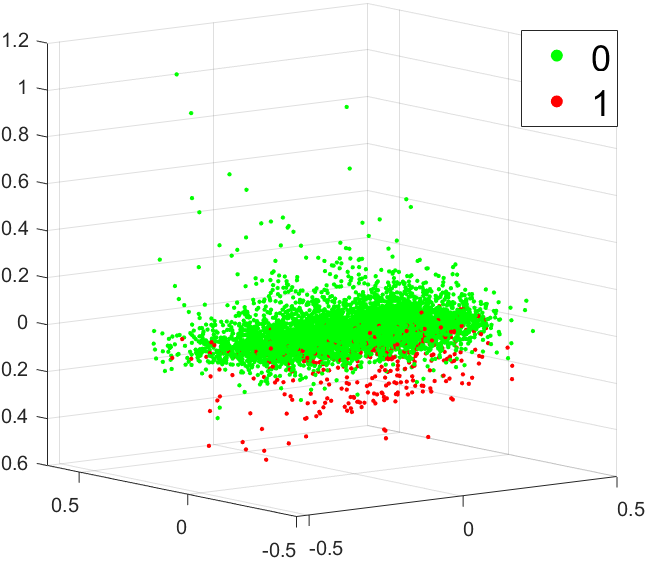}  &       \includegraphics[width=1.37in]{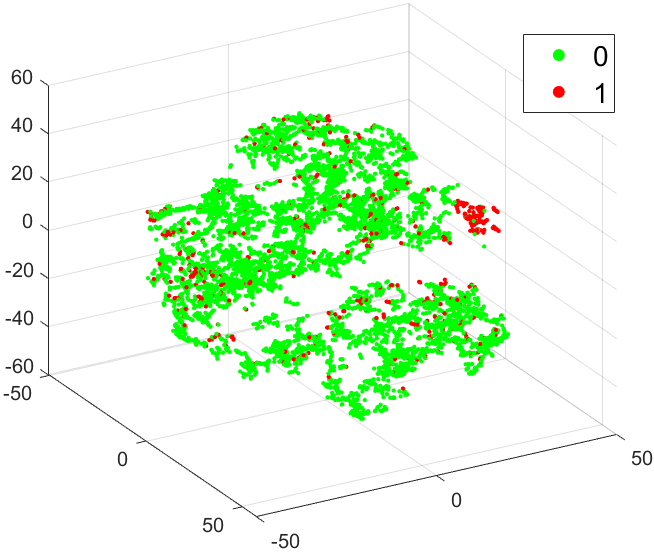}  &
        \includegraphics[width=1.37in]{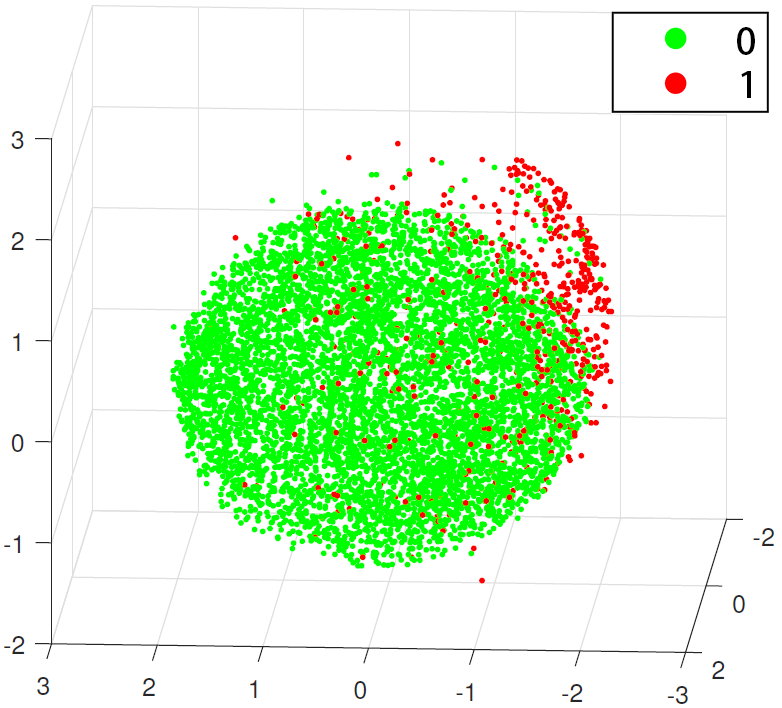} \\
      \hline 
  \end{tabular}
\label{MDSa}
\end{table}

\subsection{Run-time} 

ReScale, DScale, PCA, t-SNE and CDF-TS are all pre-processing methods, their computational complexities are shown in Table \ref{complexity}. Because many existing density-based clustering algorithms have time and space complexities of $\mathcal{O}(n^2)$, all these methods do not increase their overall complexities.

\begin{table}[!tb]
\centering
\caption{Computational complexity of  ReScale, DScale, PCA, t-SNE and CDF-TS.}
\begin{tabular}{|c|c|c|}
\hline 
Algorithm & Time complexity & Space complexity \\ 
\hline 
PCA & $\mathcal{O}(dn^2+d^3)$ & $\mathcal{O}(dn)$ \\
t-SNE  & $\mathcal{O}(dn^2)$ &  $\mathcal{O}(n^2)$ \\
ReScale & $\mathcal{O}(dn\psi)$ & $\mathcal{O}(dn+d\psi)$ \\ 
DScale and CDF-TS & $\mathcal{O}(dn^{2})$ & $\mathcal{O}(dn+n^2)$ \\ 
\hline 
\end{tabular} 
\label{complexity}
\end{table}

Table \ref{runtime} shows the runtime of the dissimilarity matrix calculation for each of the three methods on the Mfeat, Segment, Pageblocks and AnnThyroid datasets. Their parameters are set to be the same for all datasets, i.e., $\lambda=0.2$, $\psi=100$, $perplexity=30$, $NumDimensions=2$ and $\delta=0.015$. The result shows that CDF-TS had longer runtime than DScale because it requires multiple DScale calculations. However, t-SNE took the longest runtime because of search and optimisation algorithms used. We omit GPU results of PCA and t-SNE because their available program cannot take the advantage of GPU on the Matlab platform.

  \begin{table}[!hbt]
 \setlength{\tabcolsep}{2.8pt}
    \centering
    \caption{Runtime comparison of dissimilarity matrix calculation (in seconds).}
      \begin{tabular}{|c|ccccc|ccc|}
      \hline
      \multirow{2}[4]{*}{Dataset} & \multicolumn{5}{c|}{CPU}              & \multicolumn{3}{c|}{GPU}  \\
\cline{2-9}            & ReScale & DScale & PCA   & t-SNE & CDF-TS & ReScale & DScale & CDF-TS  \\
      \hline
      Mfeat & 0.69  & 0.03  & 0.05  & 2.47  & 1.26  & 0.05  & 0.02  & 2.89   \\
      Segment & 0.29  & 0.28  & 0.01  & 23.38  & 2.68  & 0.02  & 0.05  & 0.11  \\
      Pageblocks & 0.16  & 1.64  & 0.01  & 43.71  & 18.64  & 0.02  & 0.17  & 2.84  \\
      AnnThyroid & 0.19  & 3.09  & 0.01  & 61.52  & 24.40  & 0.02  & 0.26  & 2.97  \\
      \hline
      \end{tabular}%
    \label{runtime}%
  \end{table}%

\section{Discussion}
 \label{sec_disc}
  
\subsection{Parameter settings} 

Both DScale and CDF-TS are preprocessing methods and have one critical parameter $\lambda$ to define the $\lambda$-neighbourhood. The density-ratio based on a small $\lambda$ will approximate density ratio of 1 everywhere and provides no information; while on a large $\lambda$, it will approximate the actual density and has no advantage. Generally, in practice, we found that $\lambda \in [0.1, 0.3]$ and $\epsilon$ in DBSCAN and DP shall be set slightly smaller than $\lambda$. The $\lambda$ could be larger for high-dimensional and sparse datasets to include sufficient points in the $\lambda$-neighbourhood.

The parameter $\delta$ in CDF-TS controls the number of iterations, i.e., a small $\delta$ usually results in a high number of iterations in the CDF-TS process, which yields the desired outcome: the shifted clusters of homogeneous density. In our experiments, we set $\delta=0.015$ as the default value which usually ran in about 5 iterations in CDF-TS. This significantly improves the clustering results of existing algorithms. We provide a sensitivity test about $\eta$ and the number of iterations for CDF-TS on four datasets in  \ref{sensti}. 

Since CDF-TS is a preprocessing technique, we have used clustering and anomaly detection algorithms to indirectly validate the effectiveness of the transformation, i.e., it can improve the task-specific performance of existing clustering  and anomaly detection algorithms. 

However, existing density-based algorithms are sensitive to parameter settings and there is no uniform guide in setting their parameters. The sensitivity exists because the nonparametric density estimator used in these algorithms suffers from noise in a random sample without specific knowledge about the data domain \cite{gramacki2018nonparametric}. Although there are some heuristic methods, their performance cannot be guaranteed. 

Hyperparameter optimisation in a clustering algorithm, as an unsupervised learning method, is still an open question in the research community and is beyond the scope of this paper \cite{aggarwal2013data}. We followed the common practice in searching the best parameters for each algorithm to demonstrate its optimal performance.

In addition, we used a train-and-test evaluation method to compare the performance of t-SNE and CDF-TS, where a dataset is split into two subsets of equal size and the parameter setting of an algorithm is determined on the training set before it is used on the test.
The results show that CDF-TS still performs better than t-SNE on $k$NN anomaly detection tasks, i.e., CDF-TS has an average AUC of 0.90 and t-SNE has an average AUC of 0.82 with $k$NN anomaly detector. However, since density-based clustering is very sensitive to the data sampling variation, both CDF-TS and t-SNE only get the average F-measure of 0.65 with DP, still better than the original DP with the average F-measure of 0.62. 


\subsection{Compared with a density equalisation method}
\label{knnkernel}
CDF-TS reduces density variance of the dataset, such that the cluster structure becomes clearer and easier to be extracted with a single density threshold. If  the density is equalised irrespective of the bandwidth selection, then the cluster density information will be totally lost; and  there will be no clusters to detect.

In the context of kernel $k$-means, $k$NN kernel  \cite{8166757} has been suggested to be a way to equalise the density of a given dataset to reduce the bias of kernel $k$-means algorithm and to improve the clustering performance on datasets with inhomogeneous cluster densities. The dissimilarity matrix generated by $k$NN kernel is binary. It can be seen as the adjacent matrix of $k$NN graph such that ``1'' means a point is in the set of $k$ nearest neighbours of another point, and ``0'' otherwise. Thus, a $k$-means clustering algorithm can be used to group points based on their $k$NN graph.  

However, $k$NN kernel cannot be applied to both density-based clustering and $k$NN anomaly detection. This is because it converts all points to have the same density in the new space regardless of a density estimator (and its bandwidth setting) used in an intended algorithm. Thus, DBSCAN will either group neighbouring clusters into a single cluster or assign all clusters to noise.  

kNN kernel doesn't work for DP either for the same reason DP links a point to another point with a higher density to form clusters in the last step of the clustering process \cite{rodriguez2014clustering}. DP would not be able to link points when all points have the same density.

For $k$NN anomaly detection, replacing the distance measure with the $k$NN kernel produces either 0 or 1 similarity between any two points. This does not work for anomaly detection.  
 
\section{Conclusions}
\label{sec_conclu}

The density bias issue has been a perennial problem for existing density-based algorithms. Different methods have been attempted, but the issue remains.

We introduce a CDF Transform-and-Shift (CDF-TS) algorithm to comprehensively address this issue for both density-based clustering and distance-based anomaly detection algorithms. This is achieved without substantial computational cost. As CDF-TS is applied as a preprocessing step, no algorithmic modification is required to these algorithms to address their density bias issue.
 
CDF-TS is the first density-ratio method that performs both the space transformation and point-shift comprehensively, i.e., (a) all locally low-density locations  in the original space become globally low-density locations in the transformed space; and (b) the volume of every point's neighbourhood in the entire multi-dimensional space is modified via point-shifting in order to homogenise cluster density.  Existing CDF transform methods such as ReScale \cite{ZHU2016983} and DScale \cite{DSCALE} achieve a much-reduced effect on the rescaled dataset because the former is a one-dimensional method and the latter is one without point-shifting (and the resultant measure is not a metric).

In addition, CDF-TS is more generic than these two CDF transform methods because the algorithm permits different density estimators. We show that either $\epsilon$-neighbourhood or $k$NN density estimators can be successfully incorporated. It is also more effective than existing methods such as t-SNE and $k$NN kernel \cite{8166757}. As a result, CDF-TS can be applied to more algorithms/tasks with a better outcome than these methods.

Through extensive evaluations, we show that CDF-TS significantly improves the performance of four existing density-based clustering algorithms and one existing distance-based anomaly detector.

It is interesting to note that other algorithms, which are not normally considered as density-based, have also been identified with a density bias issue, e.g., spectral clustering \cite{ozertem2008mean}. We think CDF-TS can be a potential solution as well. This will be explored in the future.

\appendix 
 
\section{DScale algorithm}
\label{appA}

Algorithm~\ref{Scaling} is a generalisation of the implementation of DScale [14] which can employ a density estimator with a bandwidth parameter $\lambda$. It requires one parameter $\lambda$ only.

\begin{algorithm}[!htb]
\small
	\caption{DScale($S$, $\lambda$, $d$)}
	\begin{algorithmic}[1] 
		\Require $S$ - input distance matrix ($n  \times n$ matrix); $\lambda$ - bandwidth parameter; $d$ - dimensionality of the dataset.
		\Ensure $S'$ - distance matrix after scaling.
        \State $m \leftarrow$ the maximum distance in $S$ 
        \State Initialising $n\times n$ matrix $S'$ 
		\For {$i=1$ to $n$} 
    	\State  
    	$r(x_i)=\frac{m}{\lambda}\times (\frac{\vert\mathcal{N}(x_i,s,\lambda)\vert}{n})^{\frac{1}{d}}$   
    	\State $\forall_{x_j\in \mathcal{N}(x_i,s,\lambda)} \ S'[{x_i,x_j}]=S[{x_i,x_j}]\times r(x_i) $  
    	\State $\forall_{x_j\in D\setminus \mathcal{N}(x_i,s,\lambda)} \ S'[{x_i,x_j}] =  (S[{x_i,x_j}]-\lambda)\times \frac{m-\lambda\times r(x_i)}{m-\lambda}+\lambda\times r(x_i)$ 
		\EndFor             
		\State \Return $S'$ 
	\end{algorithmic}
	\label{Scaling} 
\end{algorithm}

\section{Intuitive argument why variance reduces as stated in Observation 2}
\label{appB}
 
Let $x\in D$ be close to $y\in D$ such that $||x-y||\approx 0$. We have:
\begin{eqnarray}
\small
 s(\bar{x}',\bar{y}') &=& ||\bar{x}' - \bar{y}'|| \nonumber \\
&=& ||\frac{1}{n}\sum_{z\in D}(\frac{s'(z,x)}{s(z,x)}(x-z)- \frac{s'(z,y)}{s(z,y)}(y-z)))|| \nonumber
\end{eqnarray}
Now since $||x-y||\approx 0$, we have $\forall_{z\in D} \ s(z,x) \approx s(z,y)$ and $s'(z,x) \approx s'(z,y)$ and can simplify the above equation:
\begin{equation} 
\small
 s(\bar{x}',\bar{y}')   
 \approx ||\frac{1}{n}\sum_{z\in D} \frac{s'(z,x)}{s(z,x)}(x-y)|| = s(x,y) \frac{1}{n}\sum_{z\in D} \frac{s'(z,x)}{s(z,x)} \nonumber 
\end{equation}

We now consider separately the cases when point $z$ lies inside and outside of the neighbourhood $\mathcal{N}(x;s,\lambda)$:

(a) For points inside the neighbourhood $z\in\mathcal{N}(x;s,\lambda)$, when $pdf(z;s,\lambda)$ varies slowly over the neighbourhood of $x$ such that $pdf(z;s,\lambda) \approx pdf(x;s,\lambda)$, then we have $r(x)\approx r(z)$ and $s'(z,x)\approx s'(x,z)$. Thus, we get: 
\begin{eqnarray}
\small
\sum_{z\in \mathcal{N}(x;s,\lambda )} \frac{s'(z,x)}{s(z,x)}\approx \sum_{z\in \mathcal{N}(x;s,\lambda )} \!\! \frac{s'(x,z)}{s(x,z)}= |\mathcal{N}(x;s,\lambda ) | \times  r(x;\lambda)
\nonumber
\end{eqnarray}

(b) For points lying outside of the neighbourhood, $z\notin\mathcal{N}(x;s,\lambda)$: given a sufficient small $\lambda$ such that $|\mathcal{N}(x;s,\lambda)|$ is small and the average density $pdf(z;s,\lambda)$ is around $\frac{n}{V_m}$, then the average rescaling factor is approximately 1: ${\mathbb E}_{x\in D}[r(z;\lambda)] \approx 1$, then we have ${\mathbb E}_{z\in \mathcal{N}(x;s,\lambda)} \ [\frac{s'(z,x)}{s(z,x)}]\approx {\mathbb E}_{z\notin \mathcal{N}(x;s,\lambda)} \ [\frac{s'(z,x)}{s(z,x)}] \approx 1$. Thus, we have:
\begin{eqnarray}
\small
 \frac{s(\bar{x}',\bar{y}')}{s(x,y)}    &\approx&
  {\displaystyle \frac{1}{n}( \!\!\sum_{z\in \mathcal{N}(x;s,\lambda)} \!\!  \frac{s'(z,x)}{s(z,x)} + \!\! \sum_{z\notin \mathcal{N}(x;s,\lambda)} \!\!  \frac{s'(z,x)}{s(z,x)})}  \nonumber \\
& \approx & (\frac{|\mathcal{N}(x;s,\lambda ) |}{n}\times r(x;\lambda)
 +\frac{n-|\mathcal{N}(x;s,\lambda ) |}{n})     \nonumber
 \label{rx}
 \end{eqnarray}
The above relation can be rewritten as:
\begin{eqnarray}
\small
	\frac{ \frac{s(\bar{x}',\bar{y}')}{s(x,y)} -1}{ r(x;\lambda) -1} \approx \frac{|\mathcal{N}(x;s,\lambda ) |}{n}  \nonumber
\end{eqnarray}
As $\frac{|\mathcal{N}(x;s,\lambda ) |}{n} \in (0,1)$, the relation between $\frac{s(\bar{x}',\bar{y}')}{s(x,y)}$ and $r(x;\lambda)$ depends on whether $x$ is in a sparse or dense region:
\begin{eqnarray}
\small
r(x;\lambda) <  \frac{s(\bar{x}',\bar{y}')}{s(x,y)}  < 1 \mbox{ if} pdf(x;s,\lambda) < \frac{n}{V_m}  \label{rate-relation1}\\
r(x;\lambda) >  \frac{s(\bar{x}',\bar{y}')}{s(x,y)}  > 1 \mbox{ if} pdf(x;s,\lambda) > \frac{n}{V_m}
\label{rate-relation2}
\end{eqnarray}

When the density varies slowly in the $\mathcal{N}(x;s,\lambda)$, then $\forall_{a,b\in \mathcal{N}(x;s,\lambda), ||a-b||\approx 0} \  \frac{s(\bar{a}',\bar{b}')}{s(a,b)} \approx R$, where $R$ is a constant. Thus, density still varies slowly in $\bar{\lambda}'$-neighbourhood.   


Let $\mathcal{N}(x';s,\bar{\lambda}')$ be the new neighbourhood covering most points in $\mathcal{N}(x;s,\lambda)$ in the transformed space. Since the density still varies slowly in  $\mathcal{N}(x';s,\bar{\lambda}')$ and when $\lambda$ is sufficiently small, we have:  
\begin{eqnarray}
\small
{pdf} (\bar{x}';s,{\lambda})   &\approx&  {pdf} (\bar{x}';s,\bar{\lambda}') \nonumber \\ &\approx& \frac{|\mathcal{N}(x;s,\lambda)|}{V_{\bar{\lambda}'}} =\frac{pdf(x;s,\lambda)\times V_{\lambda}}{V_{\lambda}\times  (\frac{s(\bar{x}',\bar{y}')}{s(x,y)})^d} 
\label{eq2}
\end{eqnarray}
\noindent 
Here we have two scenarios, depending on the density of $x$:
\begin{enumerate}
   \item  
   In sparse region $pdf(x;s,\lambda)<n/V_m$,  we have $r(x;\lambda)< \frac{s(\bar{x}',\bar{y}')}{s(x,y)} <1$, as shown in Equation \ref{rate-relation1}. Using this inequality in Equation \ref{eq2}, we have:
    \[n/V_m>pdf(\bar{x}';s,\lambda)>pdf(x;s,\lambda)\]
    \item  In dense region $pdf(x;s,\lambda)>n/V_m$,  we have $r(x;\lambda)> \frac{s(\bar{x}',\bar{y}')}{s(x,y)} >1$, as shown in Equation \ref{rate-relation2}. Using this inequality in Equation \ref{eq2}, we have:
    \[n/V_m<pdf(\bar{x}';s,\lambda)<pdf(x;s,\lambda)\] 
\end{enumerate}

Therefore, given a sufficiently small $\lambda$, $|\mathcal{N}(x;s,\lambda)|$ is small
such that the $\lambda$-neighbourhood density of every point $x\in D$ varies slowly, $pdf(\bar{x}';s,\lambda)$ should be closer to $\frac{n}{V_m}$ after the shifting process. Consequently, if the average density of $pdf(\bar{x}';s,\lambda)$ remains unchanged (approximately $\frac{n}{V_m}$), we should have: $$
\mathbb{VAR}_{\bar{x}'\in D'}[pdf(\bar{x}';s,\lambda)] < \mathbb{VAR}_{x\in D}[pdf(x;s,\lambda)]$$

\section{AMI scores for DP and LGD clustering} 
 \label{appC}
 
Table \ref{tab:NMI} shows the best AMI \cite{vinh2010information} of DP, LGD and their ReScale, DScale and CDF-TS versions. The results show similar trends as in Table \ref{best} in Section \ref{bestClu}, i.e., the CDF-TS version is significantly better than the contender versions. 

   \begin{table}[!tb]
     \centering
     \caption{The best AMI of DP and LGD, and their ReScale, DScale and CDF-TS versions. For each clustering algorithm, the best performer in each dataset is boldfaced.}
      \begin{tabular}{|c|cccc|cccc|}
      \hline
      \multirow{2}[3]{*}{Data} & \multicolumn{4}{c|}{DP}       & \multicolumn{4}{c|}{LGD}  \\
\cline{2-9}            & Orig  & ReS   & DS    & CDF-TS & Orig  & ReS   & DS    & CDF-TS   \\ \hline
      Pendig & 0.78  & 0.81  & 0.78  & \textbf{0.83} & 0.80  & 0.84  & 0.83  & \textbf{0.85} \\
      Segment & 0.75  & 0.75  & 0.76  & \textbf{0.79} & 0.67  & 0.70  & 0.67  & \textbf{0.73} \\
      Mice  & \textbf{1.00} & \textbf{1.00} & \textbf{1.00} & \textbf{1.00} & \textbf{1.00} & \textbf{1.00} & \textbf{1.00} & \textbf{1.00} \\
      Biodeg & 0.15  & 0.18  & 0.16  & \textbf{0.22} & 0.10  & 0.15  & 0.13  & \textbf{0.16} \\
      ILPD  & 0.07  & 0.09  & 0.08  & \textbf{0.12} & 0.05  & 0.117 & 0.06  & \textbf{0.12} \\
      ForestType & 0.47  & 0.58  & 0.49  & \textbf{0.62} & 0.53  & 0.57  & 0.54  & \textbf{0.60} \\
      Wilt  & 0.04  & 0.06  & 0.06  & \textbf{0.33} & 0.01  & 0.03  & 0.01  & \textbf{0.18} \\
      Musk  & 0.13  & 0.13  & 0.14  & \textbf{0.15} & 0.05  & 0.05  & 0.08  & \textbf{0.085} \\
      Libras & 0.66  & 0.68  & 0.67  & \textbf{0.69} & 0.56  & 0.58  & 0.56  & \textbf{0.63} \\
      Dermatology & 0.88  & \textbf{0.95} & 0.88  & 0.94  & 0.86  & \textbf{0.95} & 0.90  & 0.94 \\
      Haberman & 0.06  & 0.09  & 0.07  & \textbf{0.10} & 0.02  & \textbf{0.09} & 0.03  & 0.085 \\
      Seeds & 0.74  & 0.76  & 0.74  & \textbf{0.80} & 0.73  & 0.73  & 0.73  & \textbf{0.74} \\
      Wine  & 0.80  & 0.82  & 0.85  & \textbf{0.86} & 0.76  & 0.80  & 0.77  & \textbf{0.86} \\ \hdashline
      3L    & 0.57  & 0.62  & 0.64  & \textbf{0.68} & \textbf{0.47} & 0.53  & \textbf{0.47} & 0.45 \\ 
      4C    & 0.74  & 0.80  & \textbf{0.86} & 0.79  & 0.86  & 0.88  & 0.86  & \textbf{0.89}  \\
      \hline
      Average & 0.52  & 0.55  & 0.55  & 0.59 & 0.50  & 0.53  & 0.51  & 0.55  \\
      \hline
      \#Top 1 & 1     & 2     & 2     & 13    & 2     & 3     & 2     & 12 \\
      \hline
       \end{tabular}%
     \label{tab:NMI}%
   \end{table}%

\section{Sensitivity analyses}
\label{sensti}
 
In this section, we select four datasets (4C, Haberman, Wilt and ForestType) in which CDF-TS significantly improves the performance of DBSCAN for the sensitivity analyses. 

Figure \ref{Sensitivity} shows the best F-measure of the CDF-TS versions of four clustering algorithms with $\eta=0.3$ but different iterations from 1 to 10. The results show that only a few iterations allow most clustering algorithms to reach a good performance; while higher iterations may degrade the clustering performances of LGD and DBSCAN, as shown in Wilt and ForestType.

Figure \ref{Sensitivity2} shows the best F-measure of the CDF-TS versions of four clustering algorithms with 5 iterations but different $\eta$. Most clustering algorithms can reach a good F-measure with $\eta =$ 0.2 or 0.3.
 
 \begin{figure}[!htb]
  	\centering
  \begin{subfigure}[b]{0.45\textwidth}
  \centering
    \includegraphics[width=2.3in]{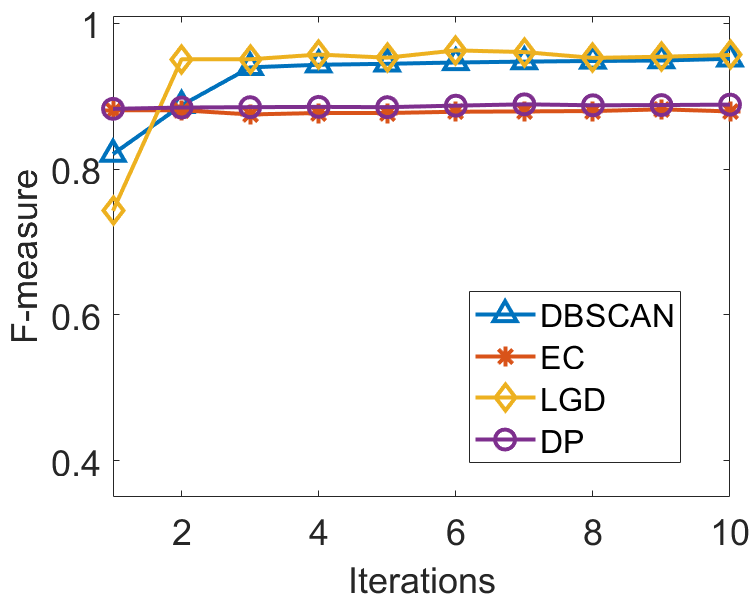}
    \caption{4C}   
    \label{Sensitivity:a}
  \end{subfigure}  
  \begin{subfigure}[b]{0.45\textwidth}
  \centering
    \includegraphics[width=2.3in]{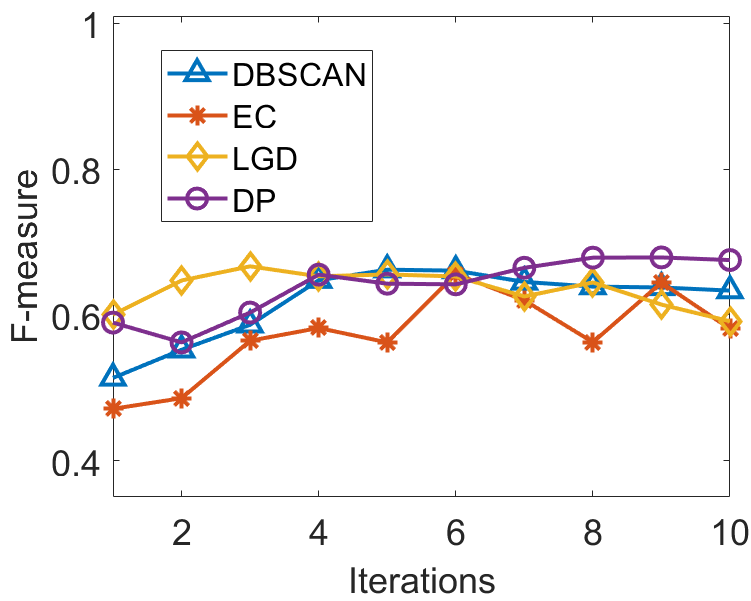}
    \caption{Haberman}
        \label{Sensitivity:b}
  \end{subfigure}
  \begin{subfigure}[b]{0.45\textwidth}
  \centering
    \includegraphics[width=2.3in]{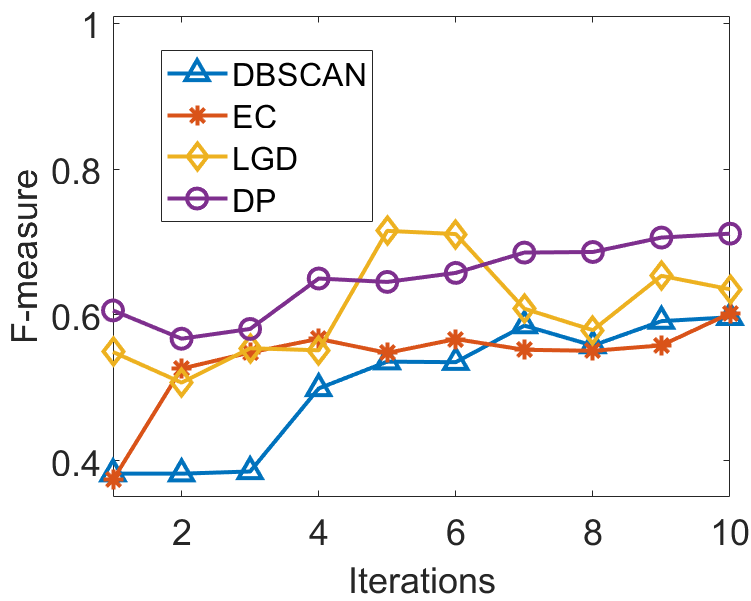}
    \caption{Wilt}
        \label{Sensitivity:c}
  \end{subfigure}
  \begin{subfigure}[b]{0.45\textwidth}
  \centering
    \includegraphics[width=2.3in]{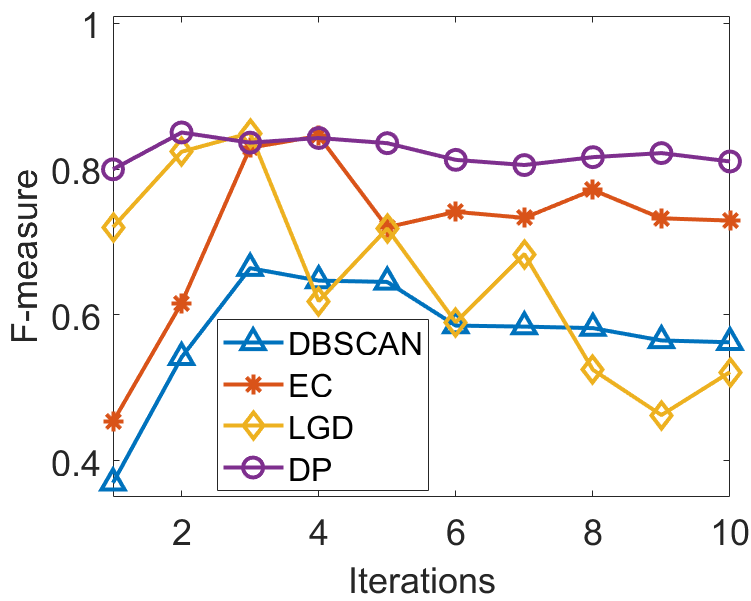}
    \caption{ForestType}
        \label{Sensitivity:d}
  \end{subfigure}
  	\caption{The best F-measure of the CDF-TS versions of four clustering algorithms with $\eta=0.3$ but different iterations from 1 to 10 on four datasets.} 
  	\label{Sensitivity}
  \end{figure}
  
  \begin{figure}[!htb]
  	\centering
  \begin{subfigure}[b]{0.45\textwidth}
  \centering
    \includegraphics[width=2.3in]{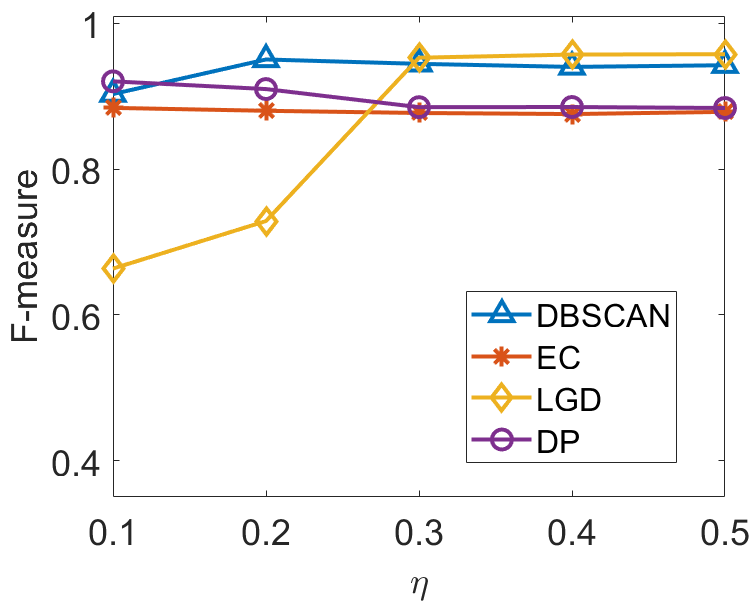}
    \caption{4C}  
    \label{Sensitivity2:a}
  \end{subfigure}  
  \begin{subfigure}[b]{0.45\textwidth}
  \centering
    \includegraphics[width=2.3in]{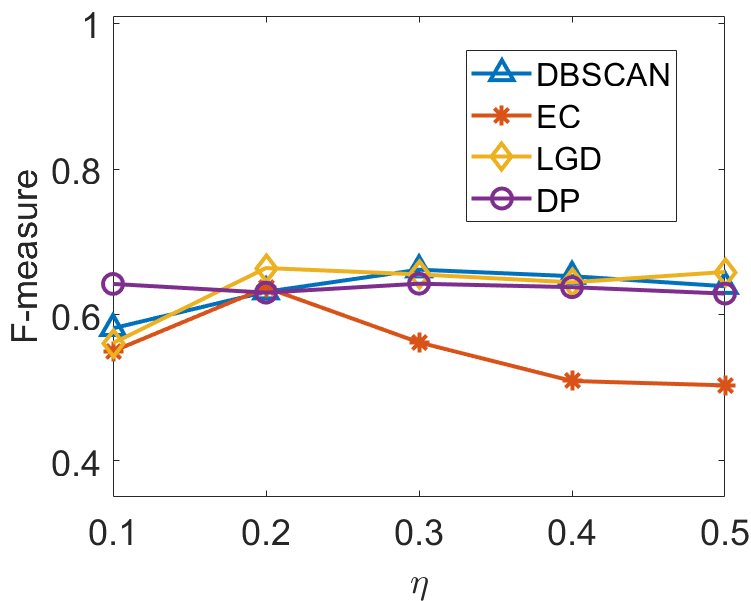}
    \caption{Haberman}
        \label{Sensitivity2:b}
  \end{subfigure}
  \begin{subfigure}[b]{0.45\textwidth}
  \centering
    \includegraphics[width=2.3in]{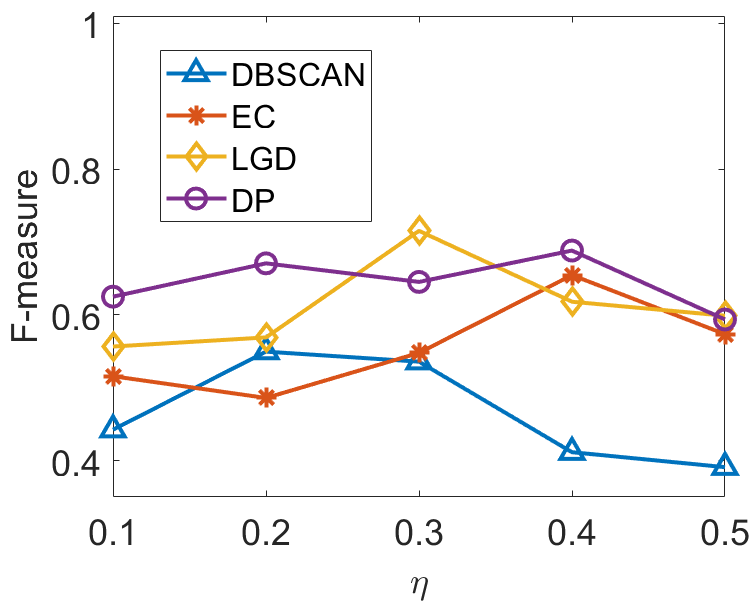}
    \caption{Wilt}
        \label{Sensitivity2:c}
  \end{subfigure}
  \begin{subfigure}[b]{0.45\textwidth}
  \centering
    \includegraphics[width=2.3in]{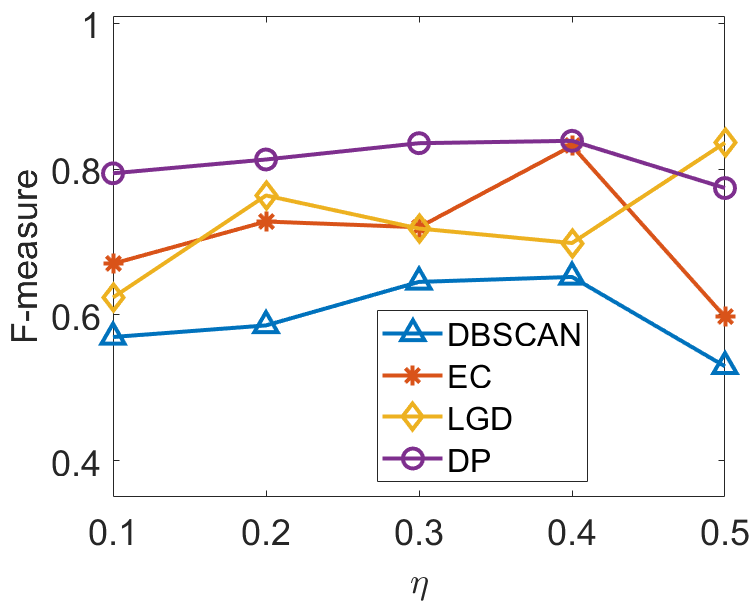}
    \caption{ForestType}
        \label{Sensitivity2:d}
  \end{subfigure}
  	\caption{The best F-measure of the CDF-TS versions of four clustering algorithms with 5 iterations but different $\eta$ settings on four datasets.} 
  	\label{Sensitivity2}
  \end{figure}

\bibliographystyle{plainnat}
\bibliography{references}

\end{document}